\documentclass[twoside,11pt]{article}

\usepackage{jmlr2e}

\ShortHeadings{Geometric Insights into SVM}{Carmichael and Marron}
\firstpageno{1}

\usepackage{amsmath}

\usepackage{graphicx}
\usepackage{caption}
\usepackage{subcaption}

\graphicspath{ {../figures/} }

\begin{document}

\title{Geometric Insights into Support Vector Machine Behavior using the KKT Conditions}

\author{\name Iain Carmichael \email iain@unc.edu \\
       \addr Department of Statistics and Operations Research\\
       University of North Carolina\\
       Chapel Hill, NC 27516, USA
       \AND
       \name \name J.S. Marron  \email marron@unc.edu \\
       \addr Department of Statistics and Operations Research\\
       University of North Carolina\\
       Chapel Hill, NC 27516, USA}

\editor{}

\maketitle

\begin{abstract}

The \textit{support vector machine} (SVM) is a powerful and widely used classification algorithm. This paper uses the Karush-Kuhn-Tucker conditions to provide rigorous mathematical proof for new insights into the behavior of SVM. These insights provide perhaps unexpected relationships between SVM and two other linear classifiers: the \textit{mean difference} and the \textit{maximal data piling direction}. For example, we show that in many cases SVM can be viewed as a cropped version of these classifiers. By carefully exploring these connections we show how SVM tuning behavior is affected by characteristics including: balanced vs. unbalanced classes, low vs. high dimension, separable vs. non-separable data. These results provide further insights into tuning SVM via cross-validation by explaining observed pathological behavior and motivating improved cross-validation methodology. Finally, we also provide new results on the geometry of \textit{complete data piling directions} in high dimensional space. 

\end{abstract}


\section{Introduction}

The \textit{support vector machine} (SVM) is a popular and well studied classification algorithm (for an overview see \citealt{scholkopf2002learning}; \citealt{shawe2004kernel}; \citealt{steinwart2008support}; \citealt{mohri2012foundations}; \citealt{murphy2012machine}). Classical classification algorithms, such as \textit{logistic regression} and \textit{Fisher linear discrimination} (FLD) are motivated by fitting a statistical distribution to the data. Hard margin SVM on the other hand is motivated directly as an optimization problem based on the idea that a good classifier should maximize the margin between two classes of separable data. Soft margin SVM balances two competing objectives; maximize the margin while penalizing points on the wrong side of the margin. 

Interpretability, explainability, and more broadly understanding why a model makes its decisions are active areas of research in machine learning \cite{guidotti2018survey, doshi2017towards}. There is a large body of research providing theoretical guarantees and computational advances for studying SVM \cite{vapnik2013nature, steinwart2008support}. Several papers have shed some light on SVM by placing it in a probabilistic framework \cite{sollich2002bayesian, polson2011data, franc2011support}. Here we take a different approach based on optimization and geometry, to understand the inner workings of SVM.

The setting of this paper is the two class classification problem. We focus on linear classifiers, but the results extend to corresponding kernel classifiers. We consider a wide range of data analytic regimes including: high vs. low dimension, balanced vs. unbalanced class sizes and separable vs. non-separable data. 

Using the \textit{KKT conditions}, this paper demonstrates novel insights into how SVM's behavior is related to a given dataset and furthermore how this behavior is affected by the tuning parameter. We discover a number of connections between SVM and two other classifiers: the \textit{mean difference} (MD) and \textit{maximal data piling classifier} (MDP). These connections are summarized in Figure \ref{fig:svm_relations}. In particular, when $C$ is small, soft margin SVM behaves like a (possibly cropped) MD classifier (Theorem \ref{thm:soft_large_C}). When the data are high dimensional, hard SVM (and soft margin with large $C$) behaves like a cropped MDP classifier (Theorem \ref{thm:hm_mdp}, Corollary \ref{cor:svm_cropped_mdp}, Theorem \ref{thm:soft_large_C}). The connection between SVM and the MD further implies connections between SVM, after a data transformation, and a variety of other classifiers such as \textit{naive Bayes} (NB) (see Section \ref{ss:md}). The connection between SVM and MDP provides novel insights into the geometry of the MDP classifier (Sections \ref{ss:data_piling_intro}, \ref{ss:data_piling_geometry_dis}). These insights explain several observed, surprising SVM behaviors which motivated this paper (Section \ref{ss:motivation}). They furthermore have applications to improving SVM cross-validation methodology and lead us to propose a modified SVM intercept term which can improve test set performance (Section \ref{s:applications}).

\begin{figure}
\centering
\textbf{Relations Between SVM and Other Classifiers}\par\medskip
\includegraphics[width=\textwidth]{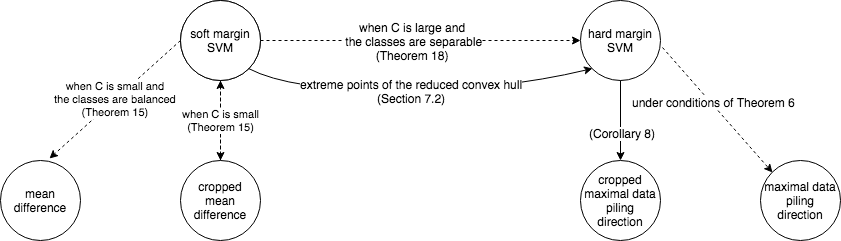}
\caption{SVM reduces to another classifier under the condition stated in the arrow. Solid line means the relation always holds. Dashed line means the relation may or may not hold depending on the data. For example, SVM reduces to the mean difference when the classes are balanced and $C$ is sufficiently small ($C \le C_{\text{small}}$) which is shown in Theorem \ref{thm:soft_small_C}.}
\label{fig:svm_relations}
\end{figure}

\subsection{Motivating Example} \label{ss:motivation}

The motivation for this paper is to understand surprising, observed SVM behavior. This section uses a simple, two dimensional example to demonstrate a number of instances of pathological or surprising SVM behavior which the rest of the paper explains and then builds upon.

Figures \ref{fig:svm_bal} and \ref{fig:svm_unbal} show the result of fitting SVM for a range of tuning parameters. The data in both figures are generated from a two dimensional Gaussian with identity covariance such that the class means are 4 apart. In Figure \ref{fig:svm_bal} the classes are balanced (20 points in each class). The data points in Figure \ref{fig:svm_unbal} are the same points as the first figure, but one additional point is added to the positive class (blue squares) so the classes are unbalanced. In both cases the classes are linearly separable. 

\begin{figure}
\centering
\begin{subfigure}[b]{0.3\textwidth}
\includegraphics[width=\textwidth]{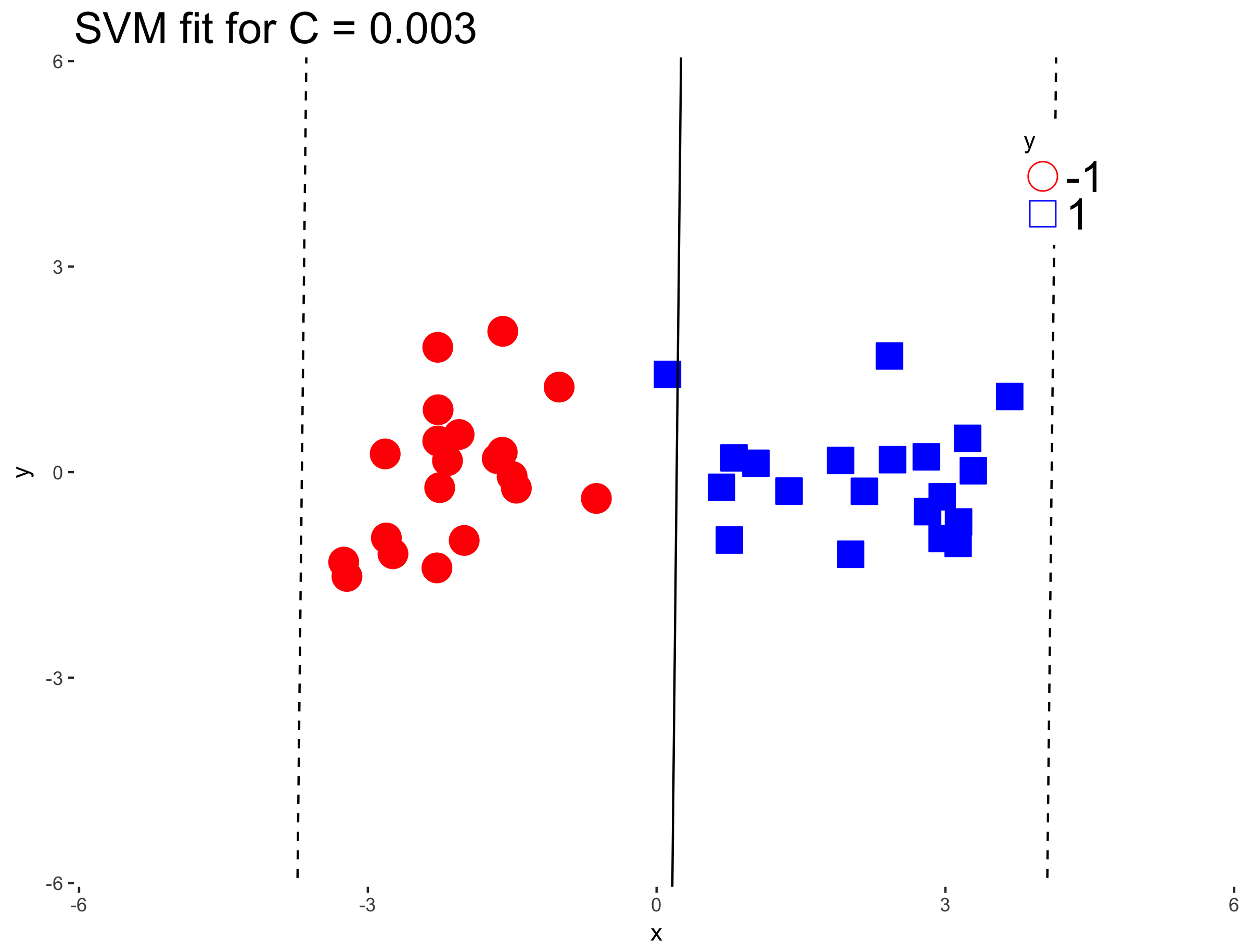}
\caption{Small C}
\label{fig:bal_large_C}
\end{subfigure}
~ 
\begin{subfigure}[b]{0.3\textwidth}
\includegraphics[width=\textwidth]{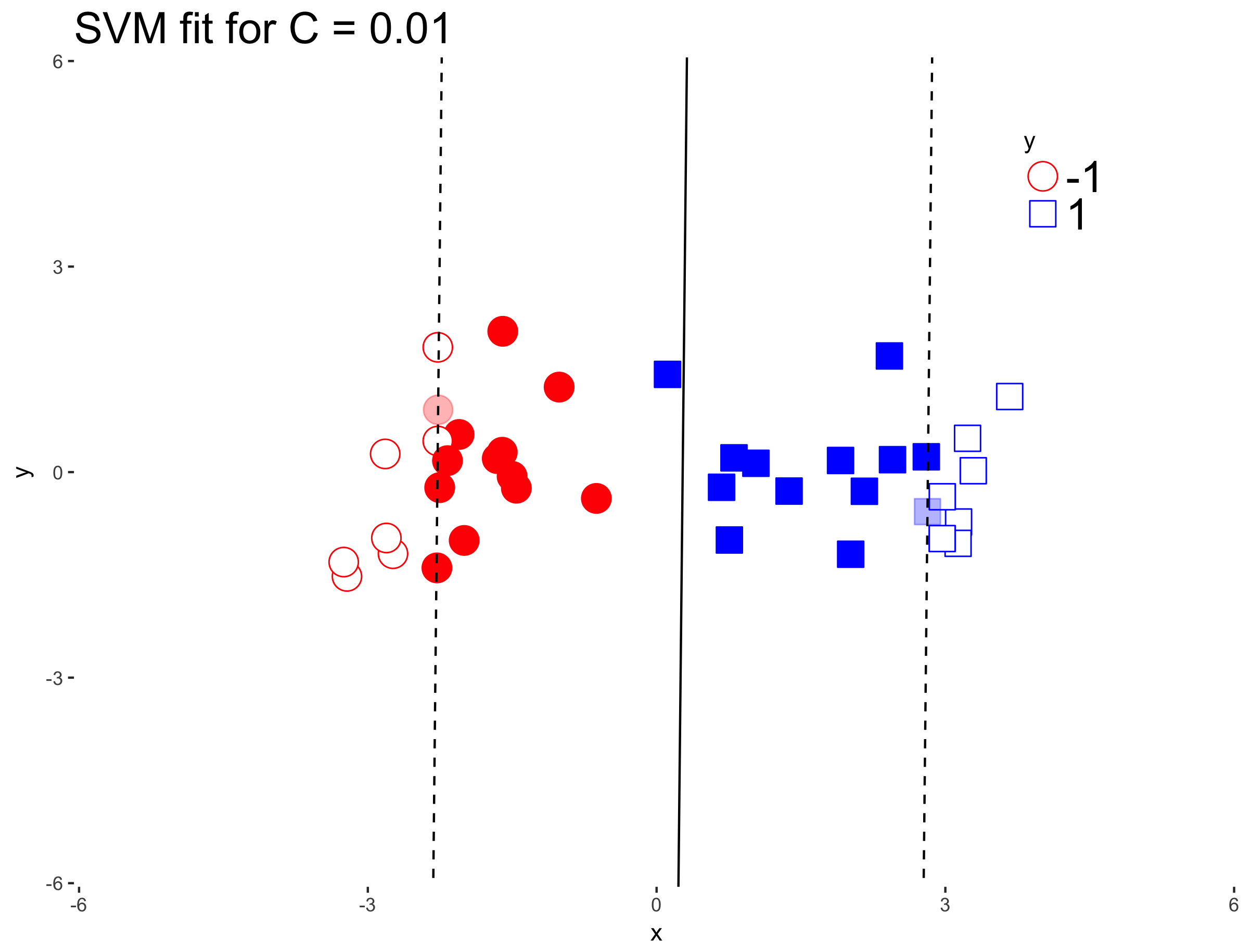}
\caption{Moderate C}
\label{fig:bal_moderate_C}
\end{subfigure}
 ~ 
\begin{subfigure}[b]{0.3\textwidth}
\includegraphics[width=\textwidth]{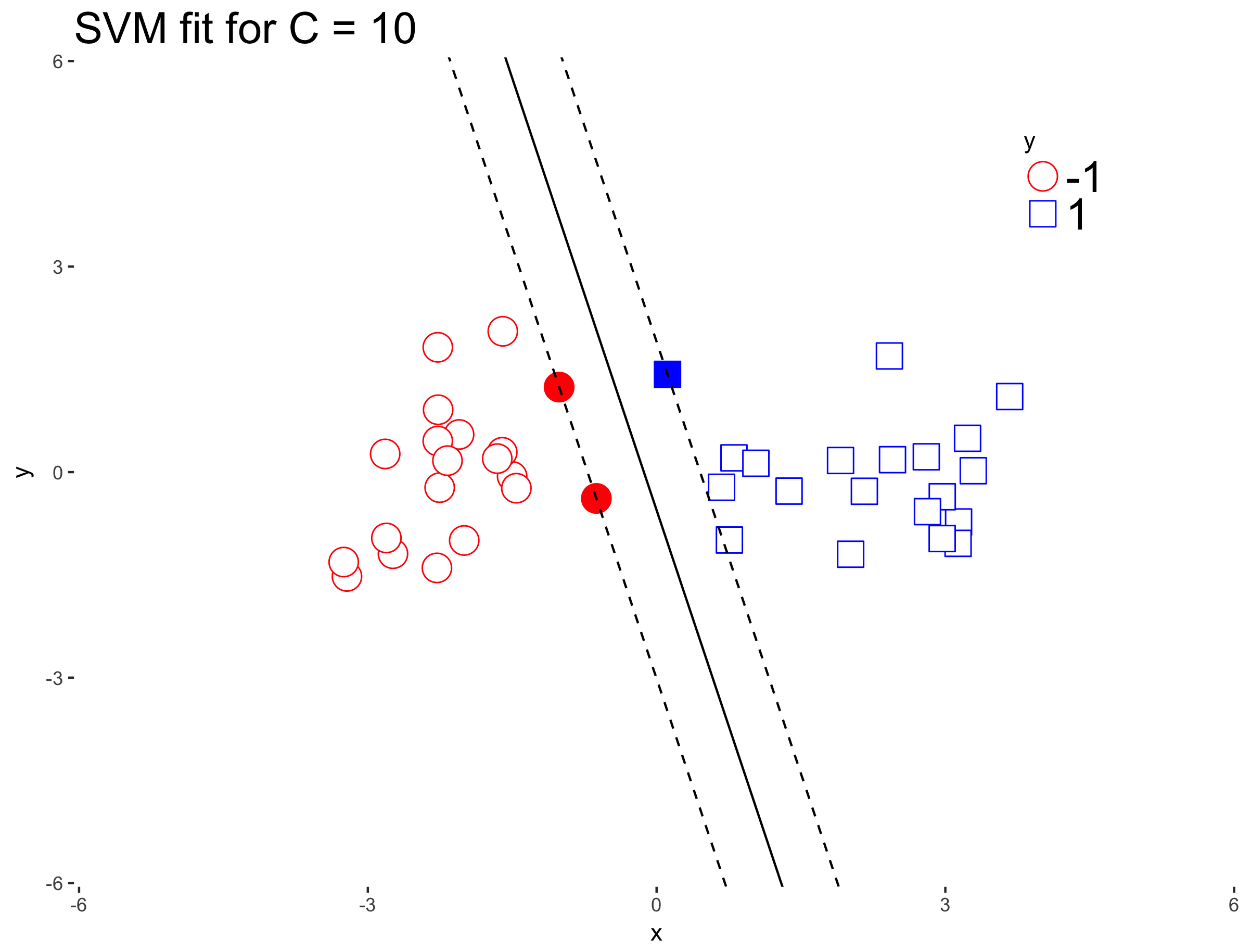}
\caption{Large C}
\label{fig:bal_small_C}
\end{subfigure}
 ~
\begin{subfigure}[b]{0.3\textwidth}
\includegraphics[width=\textwidth]{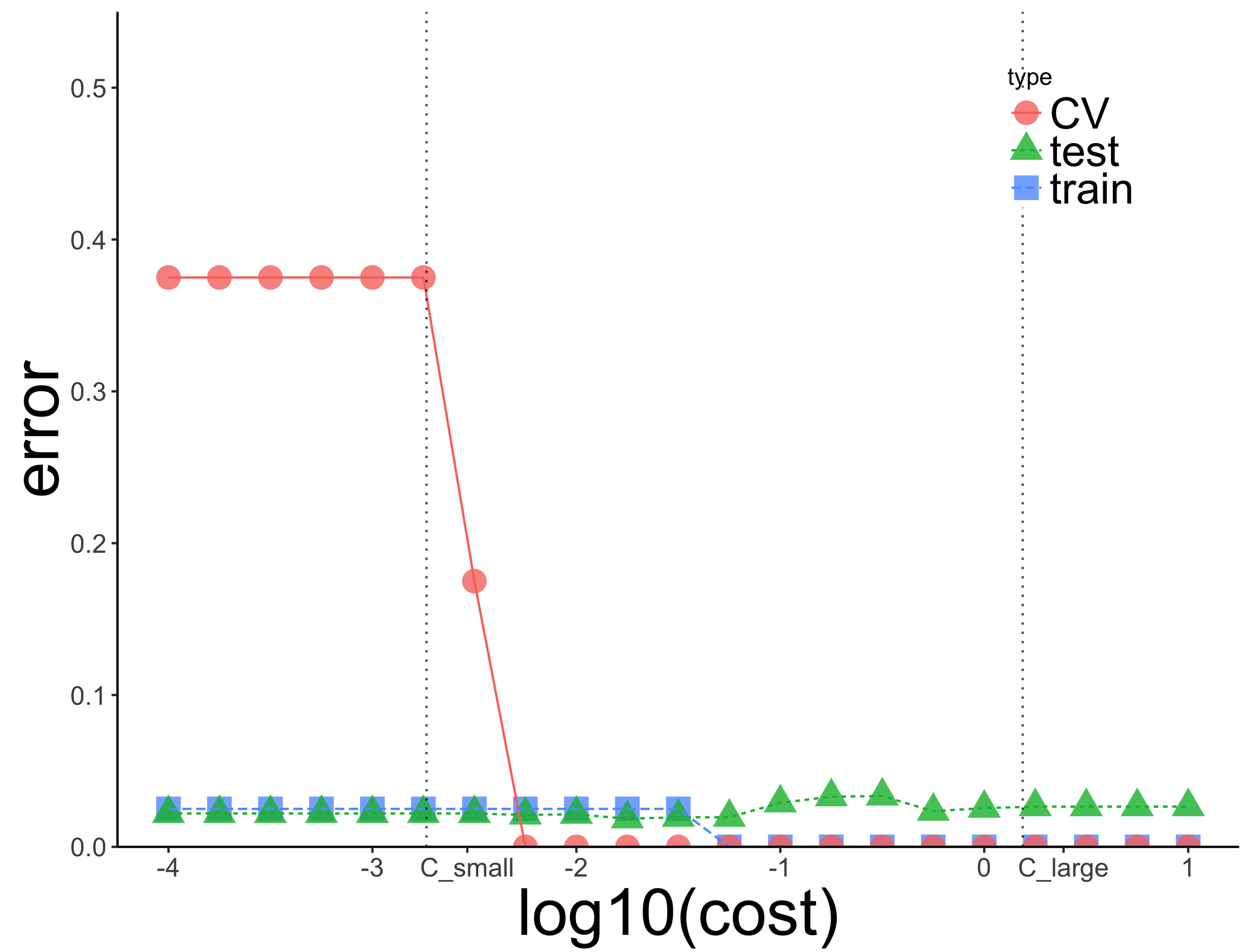}
\caption{Error rate}
\label{fig:bal_error}
\end{subfigure}
\begin{subfigure}[b]{0.3\textwidth}
\includegraphics[width=\textwidth]{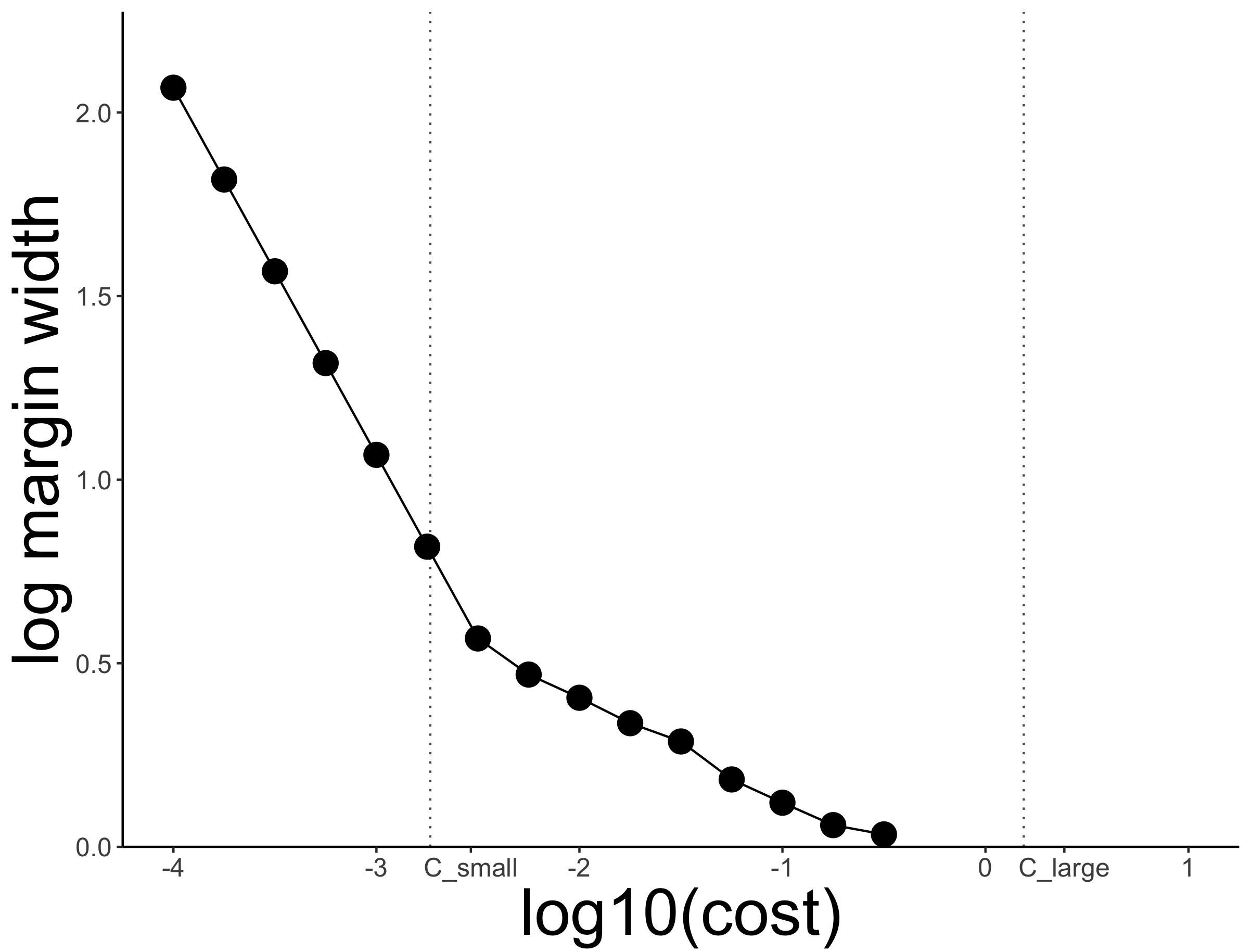}
\caption{Margin}
\label{fig:bal_margin}
\end{subfigure}
\begin{subfigure}[b]{0.3\textwidth}
\includegraphics[width=\textwidth]{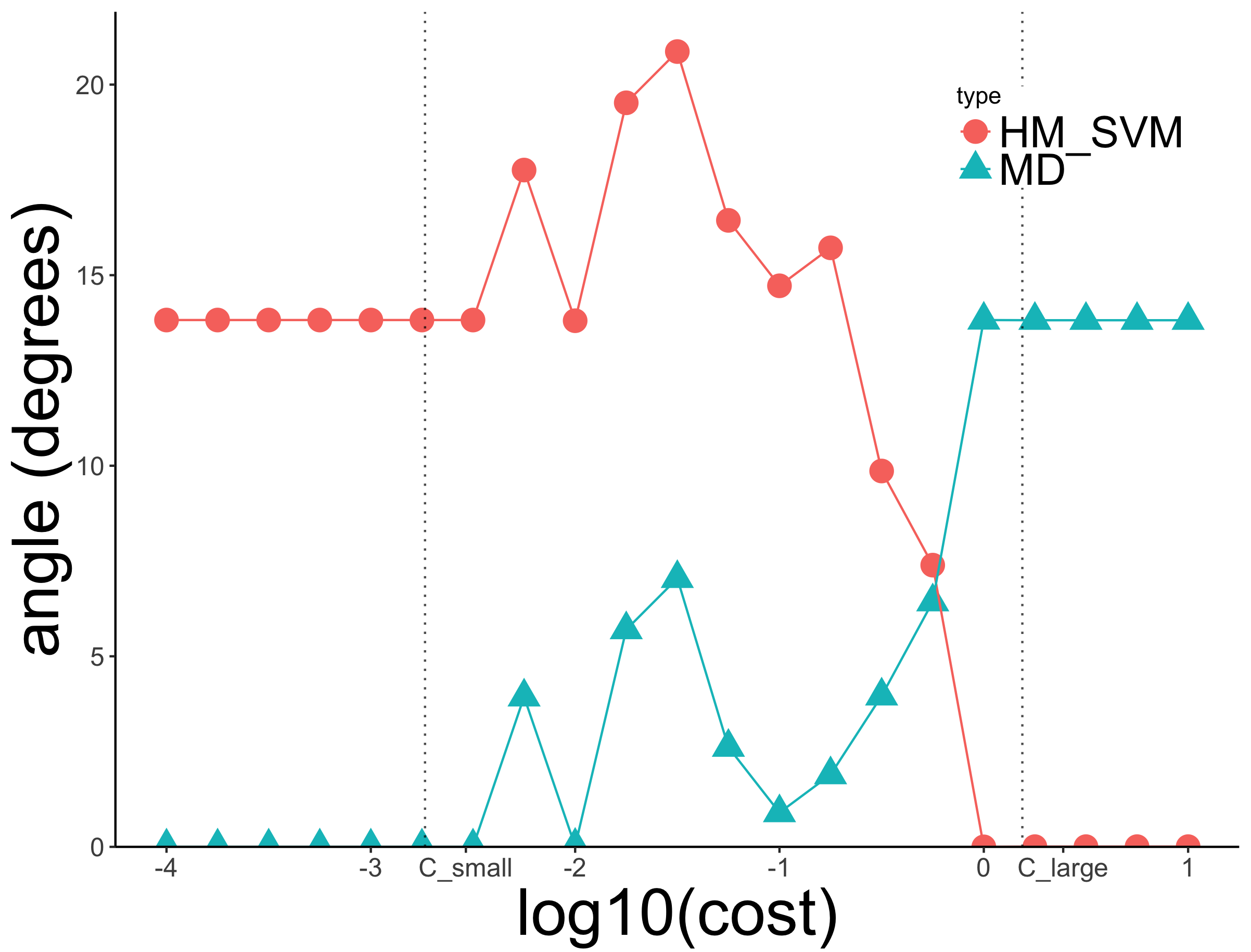}
\caption{SVM vs. MD, HM-SVM}
\label{fig:bal_angle}
\end{subfigure}
\caption{(Balanced classes) The top rows show the SVM fit for various values of $C$. The bottom row shows diagnostics which are described in the text. Figure \ref{fig:bal_error} shows that the cross-validation error curve can be very different from the training and test error. Figure \ref{fig:bal_angle}  shows that for small enough values of $C$, the SVM and MD directions are the same.}
\label{fig:svm_bal}
\end{figure}

\begin{figure}
\centering
\begin{subfigure}[b]{0.3\textwidth}
\includegraphics[width=\textwidth]{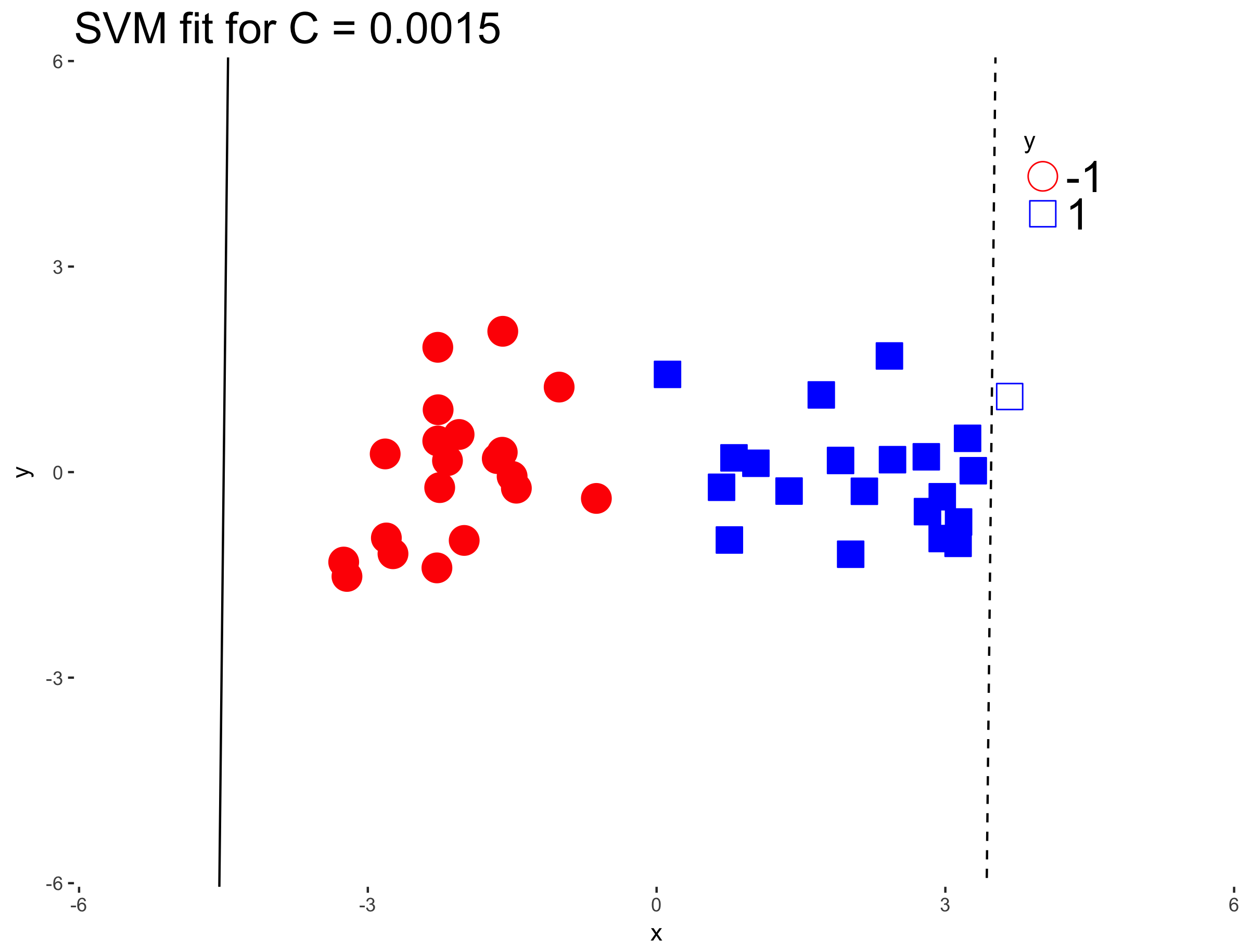}
\caption{Small C}
\label{fig:unbal_large_C}
\end{subfigure}
~ 
\begin{subfigure}[b]{0.3\textwidth}
\includegraphics[width=\textwidth]{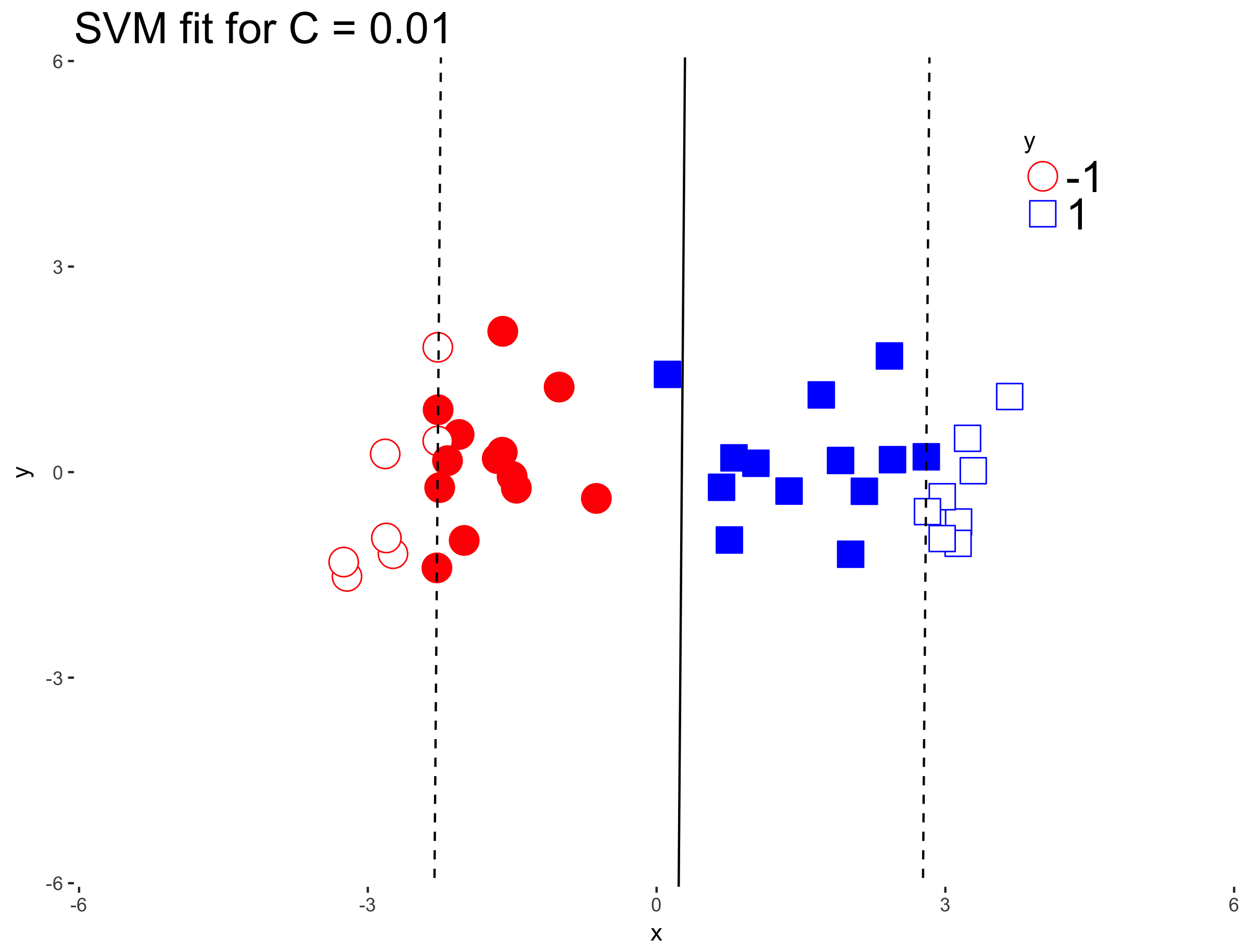}
\caption{Moderate C}
\label{fig:unbal_moderate_C}
\end{subfigure}
 ~ 
\begin{subfigure}[b]{0.3\textwidth}
\includegraphics[width=\textwidth]{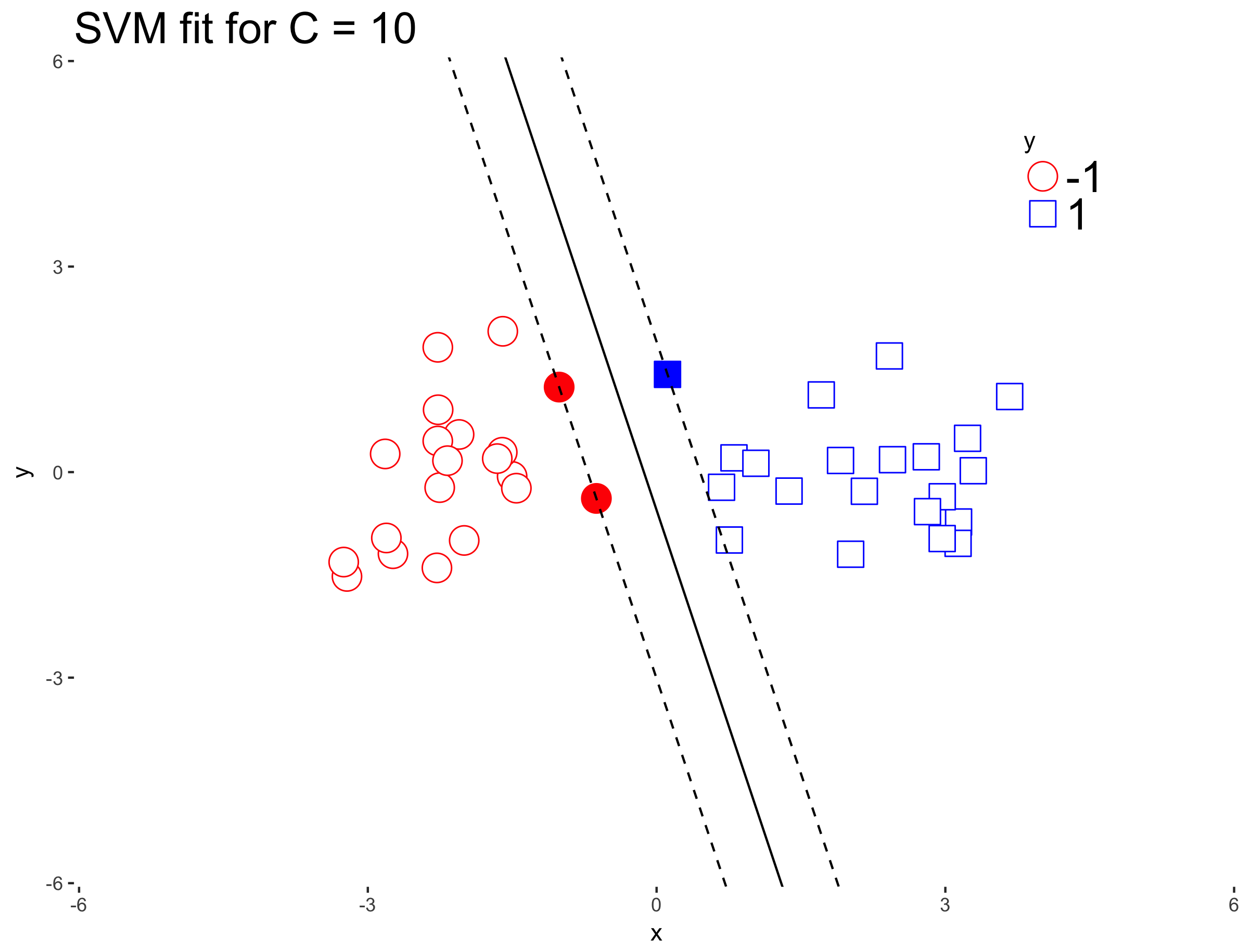}
\caption{Large C}
\label{fig:unbal_small_C}
\end{subfigure}
 ~
\begin{subfigure}[b]{0.3\textwidth}
\includegraphics[width=\textwidth]{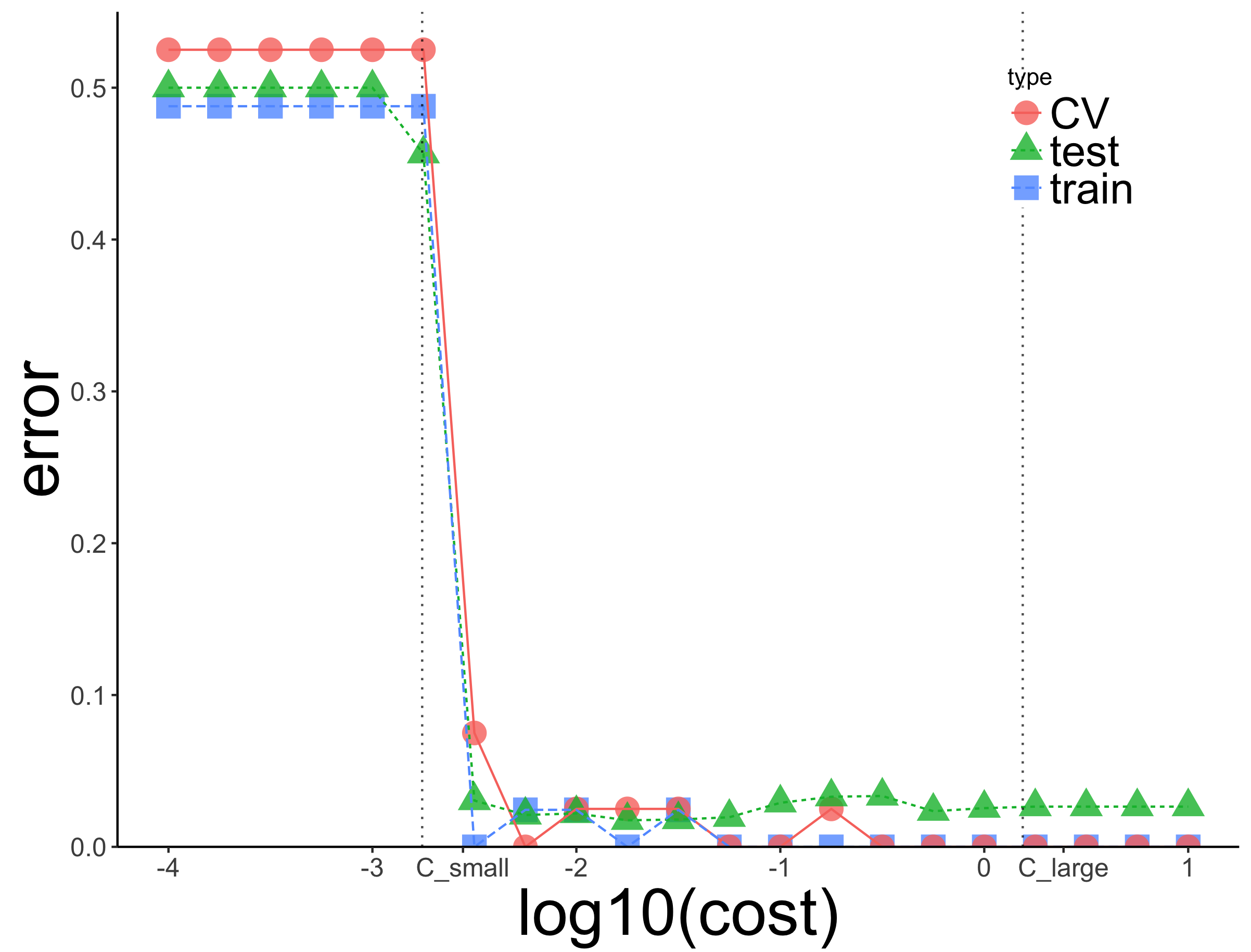}
\caption{Error rate}
\label{fig:unbal_error}
\end{subfigure}
\begin{subfigure}[b]{0.3\textwidth}
\includegraphics[width=\textwidth]{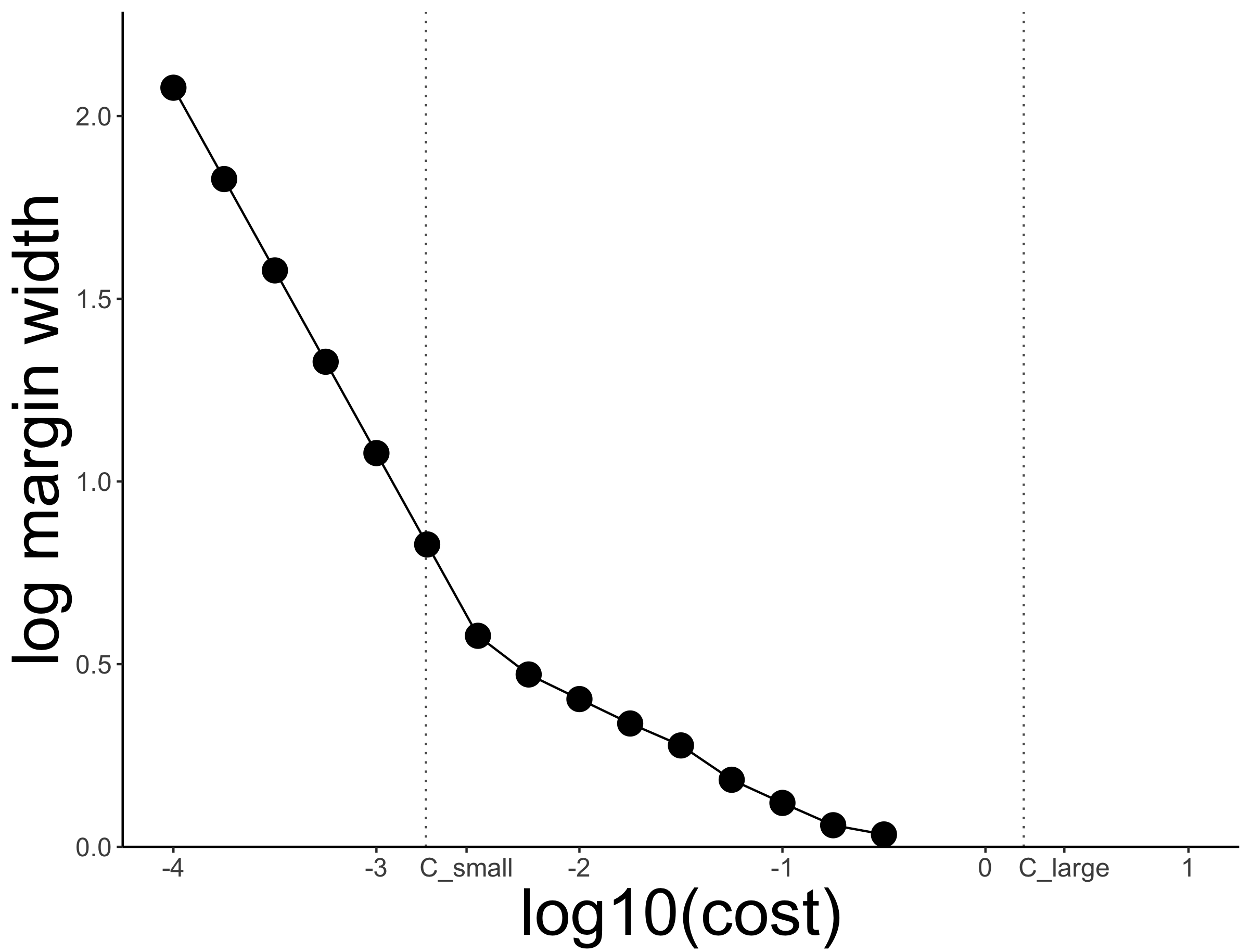}
\caption{Margin}
\label{fig:unbal_margin}
\end{subfigure}
\begin{subfigure}[b]{0.3\textwidth}
\includegraphics[width=\textwidth]{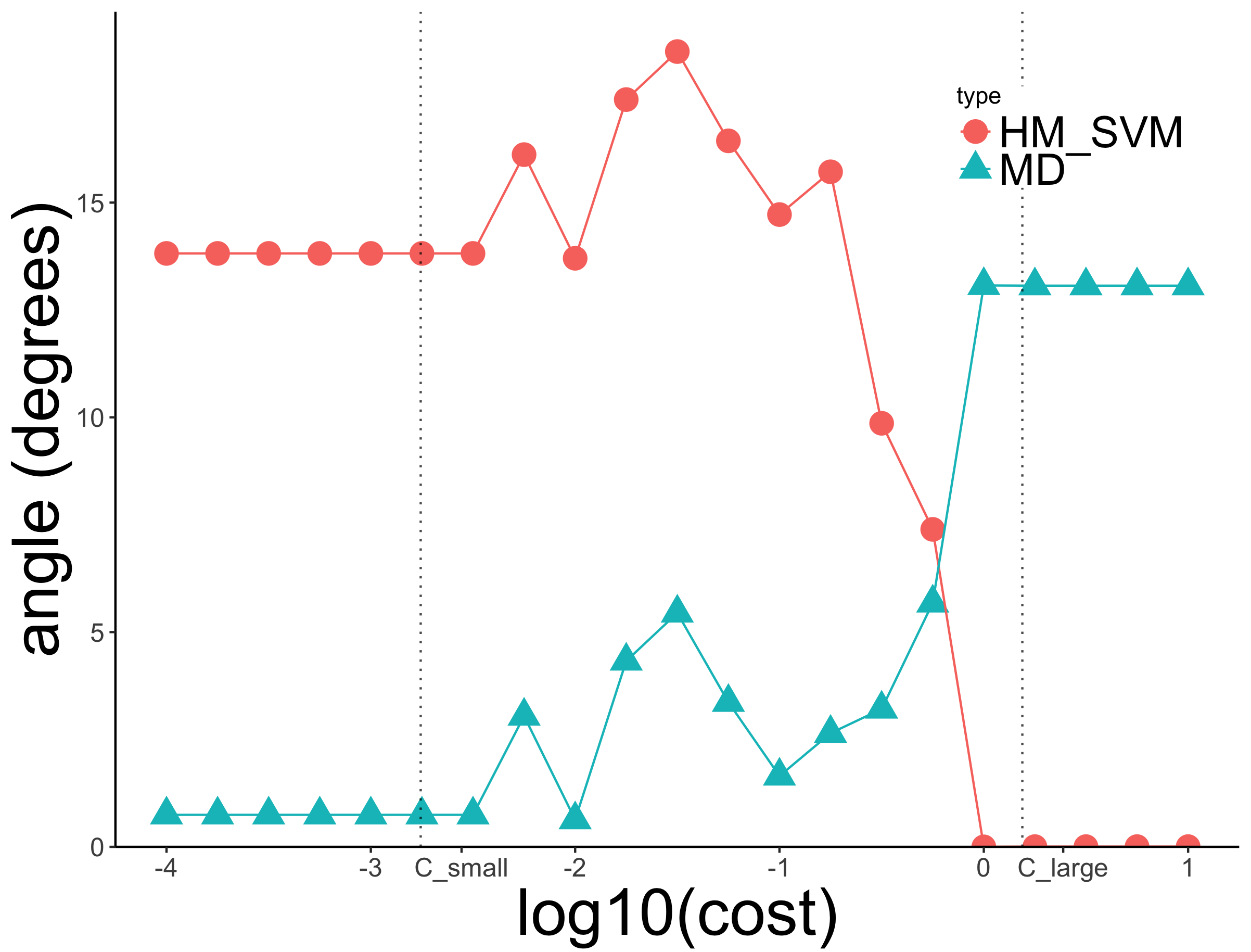}
\caption{SVM vs. MD, HM-SVM}
\label{fig:unbal_angle}
\end{subfigure}
\caption{(Unbalanced classes). The panels are the same as in Figure \ref{fig:svm_bal}, but the data now have one additional point added. When $C$ is small, the top left panel shows SVM classifies every point to the larger class (the separating hyperplane is pushed past the smaller class). For this unbalanced example the cross-validation, train and test error all behave similarly, unlike the balanced case (compare Figure \ref{fig:unbal_error} to \ref{fig:bal_error}). When $C$ is small, the angle between SVM and the MD is small but not exactly zero (compare Figure \ref{fig:unbal_angle} to \ref{fig:bal_angle}).}
\label{fig:svm_unbal}
\vspace{-.27in}
\end{figure}

The top row of panels show the data along with the SVM separating hyperplane (solid line) for three different values of $C$. The \textit{marginal hyperplanes} are shown as dashed lines and the filled in symbols are \textit{support vectors}. The bottom three panels show various functions of $C$. The bottom left panel shows three error curves: training, cross-validation (5-folds), and test set error. The bottom middle panel shows the margin width. Finally, the bottom right panel shows the angle between the soft margin SVM direction and both the hard margin SVM direction and the mean difference direction. The vertical dashed lines indicate the values of $C_{small}$ and $C_{large}$ which are discussed below. See references above or Sections \ref{ss:hm_kkt}, \ref{ss:soft_kkt} for definitions of the margin and support vectors.

Important features of these plots include:
\begin{enumerate}
\item For balanced classes (Figure \ref{fig:svm_bal}), the training, cross-validation and test set error is low for most values of $C$, then suddenly shoots up to around 50\% error for a small enough values of $C$ (see Figure \ref{fig:bal_error}). For unbalanced classes (Figure \ref{fig:svm_unbal}), this tuning error explosion for small $C$ only happens for cross validation, not the tuning or test sets (see Figure \ref{fig:unbal_error}). This pathological behavior is concerning for a number of reasons, including it demonstrates an example when performance with cross-validation may not reflect test set performance. Moreover, it is not clear why this behavior is happening.

\item Figure \ref{fig:bal_angle} show that the SVM decision boundary can be parallel to the mean difference decision boundary when the data are balanced. This behavior is surprising because the SVM optimization problem is not immediately connected to the means of the two classes. Similarly, Figure \ref{fig:unbal_angle} demonstrates an example when the SVM and MD decision boundaries are almost parallel for unbalanced classes. 

\item Both Figures \ref{fig:bal_angle}  and \ref{fig:unbal_angle} show that soft margin SVM becomes exactly equivalent to hard margin SVM for some finite value of $C$ when the data are separable.
\end{enumerate}

Theorem \ref{thm:soft_small_C} gives a complete answer to why and when the first two of these behaviors occur. Moreover, it demonstrates that this behavior occurs for every dataset. For the first example, if the data are unbalanced then the intercept term will always go off to infinity for small enough values of the tuning parameter; while SVM finds a good direction, its performance is betrayed by its intercept. For the second example, when $C$ is smaller than a threshold value $C_{\text{small}}$ (Definition \ref{def:c_small}), the SVM direction will be exactly equivalent to the MD direction when the data are balanced. Similarly, when the data are unbalanced and $C < C_{\text{small}}$ the SVM direction is close to the MD direction. In this latter case, Equations  \ref{eq:cropped_md},  \ref{eq:cropped_md_constraint} show the SVM direction must satisfy constraints that make it a cropped mean difference direction.

A formula for this threshold $C_{\text{small}}$ governing when SVM behaves like the MD is given in Definition \ref{def:c_small} as a function of the diameter of the training data. Similarly, a formula for a threshold $C_{\text{large}}$ governing when soft margin SVM becomes hard margin SVM is given in Definition \ref{def:c_large} as a function of the gap between the two training classes. These two thresholding values are shown as dotted vertical lines in the bottom three panels of Figures \ref{fig:svm_bal} and \ref{fig:svm_unbal}. 

Careful study of these behaviors, including the given formulas for the two thresholds, shows ways in which soft margin SVM's behavior can change depending on characteristics of the data including: balanced vs. unbalanced classes, whether $d \ge n - 1$, the two class diameter, whether the classes are separable and the gap between the two classes when they are separable. These results then lead to new insights into SVM tuning (Section \ref{s:applications}).

\subsection{Related Literature}\label{ss:lit_review}

\cite{hastie2004entire} show how to efficiently compute the entire SVM tuning path. While a consequence of their technical results shows that for small enough $C$, SVM behaves like the MD, they don't make the explicit connection to the MD classifier. For balanced classes they prove SVM is equivalent to the MD. For unbalanced we give a stronger, more specific characterization as a cropped MD (see Theorem \ref{thm:soft_small_C} and Lemma \ref{lem:margin_dim_bound}). Additionally, they did not find the important, general threshold values $C_{\text{small}}$ or $C_{\text{large}}$ which depend on the diameter (gap) of the data which have useful consequences for cross-validation.

Connections between SVM and other classifiers have been studied before, for example, \cite{jaggi2014equivalence} studies connections between SVM and logistic regression with an L1 penalty.

We thank the reviewers for pointing us to the nu-SVM literature \cite{scholkopf2000new, crisp2000geometric, bennett2000duality, chen2005tutorial, mavroforakis2006geometric, barbero2015geometric}. These papers re-parameterize the SVM optimization problem in a way which also provides geometric insights into the SVM solution and makes the tuning parameter more interpretable (roughly controlling the number of support vectors). Section \ref{ss:nu_svm} discusses how we can use the nu-SVM formulation and our results to provide additional insights into SVM. The nu-SVM formulation could also be used to prove some of our results (e.g. a weaker version of parts of Theorem \ref{thm:soft_small_C}) in a different way (however, we believe our proof techniques require less background work). The nu-SVM literature is mostly focused on computation and we did not find much overlap with our results.

SVM robustness properties have been previously studied (\cite{scholkopf2000new, steinwart2008support}), however, the cropped MD characterization of SVM for small $C$ appears to be new.

We find the gap and diameter (Definitions \ref{def:gap}, \ref{def:diam}) of the dataset are important quantities for SVM tuning. These quantities show up in other places in the SVM literature, for example, their ratio is an important quantity in statistical learning theory \cite{vapnik1999overview}.

Some previous papers have suggested modifying SVM's intercept \cite{crisp2000geometric}. We suggest a particular modification (Section \ref{ss:better_svm_intercept}) which addresses the \textit{margin bounce} phenomena (Section \ref{ss:margin_bounce}).  

SVM tuning has been extensively studied (\cite{steinwart2008support}[Chapter 11]). Some papers focus on computational aspects of SVM tuning e.g. cheaply computing the full tuning path \cite{hastie2004entire}. Other papers focus on tuning kernel parameters \cite{sun2010analysis}.  Some papers optimize alternative metrics which attempt to better approximate the test set error \cite{chapelle2000model, ayat2005automatic}. Some papers propose default values for tuning parameters \cite{mattera1999support, cherkassky2004practical}. Our tuning results provide different kinds of insights whose applications are discussed in more detail in Section \ref{s:applications}.


\section{Setup and Notation}\label{s:setup}

A linear classifier is defined via the \textit{normal vector} to its discriminating hyperplane and an \textit{intercept} (or \textit{offset}). A key idea in this paper is to compare \textit{directions} of linear classifiers. Comparing the direction between two classifiers means comparing their normal vector directions; we say two directions are equivalent if one is a scalar multiple of the other (see Section 2). Note that two classifiers may have the same direction, but lead to different classification algorithms (i.e. the intercepts may differ).

Suppose we have $n$ labeled data points $\{(\mathbf{x}_i, y_i) \}_{i=1}^n$  and index sets $I_+, I_-$ such that $y_i = 1$ if $i \in I_+$, $y_i = -1$ if $i \in I_-$ and $\mathbf{x}_i \in \mathbb{R}^d$. Let $n_+ = |I_+|$ and $n_- = |I_-|$ be the class sizes (i.e. $n_- + n_+ = n$). We consider linear classifiers whose decision function is given by 
$$
f(\mathbf{x}) = \mathbf{w}^T \mathbf{x} + b,
$$
where $\mathbf{w} \in \mathbb{R}^d$ is the normal vector and $b \in \mathbb{R}$ is the intercept (classification rule sign$(f(x))$). 

Given two vectors $\mathbf{v}, \mathbf{w} \in \mathbb{R}^d$ we consider their \textit{directions to be equivalent} if there exists $a \in \mathbb{R}, a \neq 0$ such that $a \mathbf{w} = \mathbf{v}$ (and we will write $\mathbf{w} \propto \mathbf{v}$). Using this equivalence relation we can quotient $\mathbb{R}^d$ into the space of directions (formally real projective space). Intuitively, this is the space of lines through the origin.

In this paper we consider the following linear classifiers: hard margin SVM, soft margin SVM (which we refer to as SVM),  mean difference (also called \textit{nearest centroid}), and the maximal data piling direction. 

Often linear classification algorithms can be extended to a wide range of non-linear classification algorithms using the \textit{kernel trick} \cite{scholkopf2002learning}. While a kernlized linear classifier is no longer linear in the original data, it is a linear classifier in some transformed space (often called the \textit{feature space}). Therefore, in this paper we focus on the linear case, but our mathematical results extend to the kernel case.

\subsection{Mean Difference and Convex Classifiers}\label{ss:md}
The \textit{mean difference} (MD) classifier selects the hyperplane that lies half way between the two class means. In particular the vector $\mathbf{w}_{md}$ is given by the difference of the class means
\begin{equation}
\label{eq:md}
\begin{aligned}
 \mathbf{w}_{md} & := \frac{1}{n_+} \sum_{i \in I_+} \mathbf{x}_i - \frac{1}{n_-} \sum_{i \in I_-} \mathbf{x}_i \\
&  := \bar{\mathbf{x}}_+ - \bar{\mathbf{x}}_-.
\end{aligned}
\end{equation}
By replacing the mean with another measure of center (e.g. the \textit{spatial median} \citealt{brown1983statistical}) we can motivate a number of other classifiers. 

We say a linear classifier is a \textit{convex classifier} if its normal vector, $\mathbf{w}$ is given as the difference of points lying in the convex hulls of the two classes (i.e. $\mathbf{w}  = \mathbf{c}_+ - \mathbf{c}_-$ where $\mathbf{c}_{\pm} \in \text{conv}(\{\mathbf{x}_i | i \in I_{\pm}\})$). These classifiers are sometimes refered to as \textit{nearest centroid} classifiers because they classify test points by assigning them to the class with the nearest centroid, $\mathbf{c}_+$ or $\mathbf{c}_-$ .

We define  \textit{convex directions}, $C$, to be the set of directions such a classifier can take.
\begin{definition}
\label{def:cvx_dir}
Let $C$ denote the set of all vectors associated with the directions that go between the convex hulls of the two classes i.e.
\end{definition}
$$C := \{ a\left( \mathbf{c}_+ - \mathbf{c}_-  \right)| a \in \mathbb{R}, a \neq 0, \text{ and } \mathbf{c}_j \in \text{conv}(\{\mathbf{x}_i \}_{i \in I_j}), j=\pm \}.$$
The set $C$ may be all of $\mathbb{R}^d$ if, for example, the two convex hulls intersect. When the data are linearly separable $C$ is a strict subset of $\mathbb{R}^d$. This set of directions will play an important role in later sections.

\subsection{Data Transformation}\label{ss:data_transform}

It is common to transform the data before fitting a linear classifier, for example, the analyst may mean center the variables then scale them by the standard deviation. A number of classifiers can be viewed as either: apply a data transformation then fit a more simple classifier (such as MD) or as a distinct classifier. These classifiers include: \textit{naive Bayes}, \textit{Fisher linear discrimination}, \textit{nearest shrunken centroid}, \textit{regularized discriminant analysis}, and more \cite{friedman2001elements}. 

For example, when $d < n -1$  the Fisher's linear discriminant direction is given by
\begin{equation}
\label{eq:fld}
\mathbf{w}_{fld} :=  \widehat{\Sigma}_{pool}^{-1}(\bar{\mathbf{x}}_+ - \bar{\mathbf{x}}_-),
\end{equation}
letting $X_-$ and $X_+$ be the data matrix for the respective classes and the \textit{pooled sample covariance} is $\widehat{\Sigma}_{pool} := \frac{1}{n-2} \left[(X_+ -\overline{X}_+)^T (X_+ -\overline{X}_+) + (X_- -\overline{X}_-)^T (X_- -\overline{X}_-)\right]$. Note the inevitability of $\widehat{\Sigma}_{pool}$ plays an important role in the next section.

It is easy to see FLD is equivalent to transforming the data by the pooled sample covariance matrix (i.e. multiplied each data point by $\widehat{\Sigma}_{pool}^{- 1/2}$) then computing the MD classifier (where we apply the same transformation to the test data). More generally, if we have a simple, convex classifier (e.g. the MD) given by $\mathbf{w}$ and we apply a data transformation in the form of $\Sigma^{-1/2}$ to the data we obtain the same classifier as $\Sigma^{-1} \mathbf{w}$. 

The technical results of this paper connect SVM to MD (and various other convex classifiers), however, they apply more generally. If the analyst first transforms the data before fitting SVM, as is common in practice, then our results connect SVM to the more general classifier. For example, naive Bayes is equivalent to first transforming the data by a certain diagonal covariance matrix; in this case, our results connect SVM to naive Bayes.

\subsection{Maximal Data Piling Direction} \label{ss:mdp_intro}
 For linear classifiers one frequently projects the data onto the one dimensional subspace spanned by the normal vector.  \textit{Data piling}, first discussed by \cite{marron2007distance}, is when multiple points have the same projection on the line spanned by the normal vector. For example, all points on SVM's margin have the same image under the projection map. \cite{ahn2010mdp} showed that when $d \ge n - 1$ there are directions such that each class is projected to a single point i.e. there is \textit{complete data piling}. 
\begin{definition}
\label{def:cdp_directions}
A vector $\mathbf{w} \in \mathbb{R}^d$ gives complete data piling for two classes of data if there exist $a, b \in \mathbb{R}$, with $a \neq 0$ such that 
$$\mathbf{w}^T \mathbf{x}_i = a y_i + b \text{ for each } i=1, \dots, n,$$
where $b$ is the midpoint of the projected classes and $a$ is half the distance between the projected classes. 
\end{definition}
The \textit{maximal data piling} (MDP) direction, as its name suggests, searches around all directions of complete data piling and finds the one that maximizes the distance between the two projected class images. This classifier has been studied in a number of papers such as \cite{ahn2012clustering}, \cite{lee2013hdlss}, and \cite{ahn2010mdp}. The MDP direction can be computed analytically
\begin{equation}
\label{eq:mdp}
\mathbf{w}_{mdp} = \widehat{\Sigma}^-(\bar{\mathbf{x}}_+ - \bar{\mathbf{x}}_-),
\end{equation}
where $A^-$ is the Moore-Penrose inverse of a matrix $A$ and $\widehat{\Sigma} := \frac{1}{n-1} (X -\bar{X})^T (X -\bar{X})$ is the \textit{global sample covariance} matrix (in contrast with the pooled sample covariance of FLD given above).

The MDP direction has an interesting relationship to Fisher linear discrimination. Recall the formula for FLD show in Equation \ref{eq:fld} above. \cite{ahn2010mdp} showed that in low dimensional settings FLD and the MDP formula are the same (though in low dimensional settings MDP does not give complete data piling); when $d < n -1$ the above two equations are equivalent.

Another view of this relation comes from the optimization perspective. FLD attempts to find the direction that maximizes the ratio of the projected ``between-class variance to the within-class variance," \cite{bishop2006pattern}. This problem is well defined only in low dimensions; in high dimensions when $d \ge n -1$ there exist directions of complete data piling where the within class projected variance is zero. In the high dimensional setting MDP searches around these directions of zero within class variance to find the one that maximizes the distance between the two classes (i.e. the between-class variance). 

\subsection{Support Vector Machine}
Hard margin support vector machine is only defined when the data are linearly separable; it seeks to find the direction that maximizes the margin separating the two classes. It is defined as the solution to the following optimization problem,

\begin{equation}
\label{eq:svm_hard}
\begin{aligned}
& \underset{\mathbf{w} \in \mathbb{R}^d, b \in \mathbb{R}}{\text{minimize}}  & &  \frac{1}{2}||\mathbf{w}||^2 \\
& \text{subject to} & &    y_i(\mathbf{x}_i \cdot \mathbf{w} + b) \ge 1, \text{ for } i = 1, \dots, n.
\end{aligned}
\end{equation}

When the data are not separable Problem (\ref{eq:svm_hard}) can be modified to give soft margin SVM by adding a tuning parameter $C$ and slack variables $\xi_i$  which allow points to be on the wrong side of the margin,

\begin{equation}
\label{eq:svm_soft}
\begin{aligned}
& \underset{\mathbf{w} \in \mathbb{R}^d, b \in \mathbb{R}}{\text{minimize}}  & &  \frac{1}{2}||\mathbf{w}||^2 + C \sum_{i} \xi_i\\
& \text{subject to} & &   y_i(\mathbf{x}_i \cdot \mathbf{w} + b) \ge 1- \xi_i, \text{ for } i = 1, \dots, n\\
& & & \xi_i \ge 0, \text{ for } i = 1, \dots, n.
\end{aligned}
\end{equation}
For a detailed introduction to SVM see \cite{mohri2012foundations}.

In both cases the direction is a linear combination of the training data points
 $$\mathbf{w}_{svm}  = \sum_{i \in I_+} \alpha_i  \mathbf{x}_i  -  \sum_{i \in I_-} \alpha_i  \mathbf{x}_i.$$
It turns out this linear combination always gives a direction that points between the convex hull of the two classes (see Definition \ref{def:cvx_dir}).


\section{Hard Margin SVM in High Dimensions} \label{s:hmsvm-hd}
In this section we provide novel insights into the geometry of complete data piling  which are then used to characterize the relationship between hard margin SVM and MDP in high dimensions. The results are stated in the first two subsections then proved in the remaining two subsection and appendix.  

For this section we assume $d \ge n -1$.  We further assume the data are in general position and separable, which implies the data are linearly independent if $d \ge n$ and affine independent if $d = n -1$. The data are in general position with probability 1 if they are generated by an absolutely continuous distribution in high dimensions. Typically the phenomena studied here happens in the $n - 1$ dimensional affine space generated by the data.


\subsection{Complete Data Piling Geometry}\label{ss:data_piling_intro}
Define the set  $P$ of \textit{complete data piling directions} using ideas from Definition \ref{def:cdp_directions}.
\begin{definition}
Let $P$ denote the vectors associated with directions that give complete data piling i.e.
\end{definition}
$$P := \{\mathbf{v} \in \mathbb{R}^d | \exists a, b \in \mathbb{R}, a \neq 0 \text{ s.t. } \mathbf{v}^T\mathbf{x}_i = a\cdot y_i + b \text{ for each } i=1, \dots, n \}.$$
Note the set of complete data piling directions can be empty, however, if the data are in general position then $P \neq \emptyset$ when $d \ge n -1$. In this case, \cite{ahn2010mdp} point out there are infinitely many of such directions in the ($n$ dimensional) subspace generated by the data that give complete data piling; in fact there is a great circle of directions in this subspace (if we parameterize directions by points on the unit sphere). 

Theorem \ref{thm:affine_piling_directions} shows there is a single complete data piling direction that is also within the ($n-1$ dimensional) affine hull of the data. The remaining directions in $P$ come from linear combinations of this unique direction in the affine hull and any vector normal to that hull. 
\begin{theorem}
\label{thm:affine_piling_directions}
The set of complete data piling directions, $P$, intersects the affine hull of the data in a single direction which is the maximal data piling direction.
\end{theorem}
Theorem \ref{thm:affine_piling_directions} is proved in the appendix. 

\subsection{Hard Margin SVM and Complete Data Piling}\label{ss:hm_cdp}

A simple corollary of Theorem \ref{thm:affine_piling_directions} is:
\begin{corollary}
\label{bp_intersection}
The intersection of the convex directions, $C$, and the complete data piling directions, $P$, is either empty or a single direction i.e. 
$$C \cap P = \emptyset \text{ or } C \cap P  = \{a \mathbf{v} | a \in \mathbb{R}\}.$$
\end{corollary}
In other words, if a convex classifier gives complete data piling then it has to also be the MDP; furthermore, there can be at most one convex classifier which gives complete data piling.

The core results for hard margin SVM are summarized in the following theorem. Note that this theorem also characterizes when SVM has complete data piling
\begin{theorem} 
\label{thm:hm_mdp}
The hard margin SVM and MDP directions are equivalent if and only if there is a non-empty intersection between the convex directions, $C$, and the complete data piling directions, $P$. In this case, the intersection is a single direction which is the hard margin SVM direction and the MDP direction i.e. 
$$\mathbf{w}_{hm-svm} \propto \mathbf{w}_{mdp} \iff P \cap C \neq \emptyset \iff \mathbf{w}_{hm-svm} \propto \mathbf{w}_{mdp} = C \cap P$$
\end{theorem}
Where we use the equality sign to indicated $C \cap P$ is a single direction. Theorem \ref{thm:hm_mdp} is a consequence of Corollary \ref{bp_intersection}, Lemma \ref{svm_data_piling}, Lemma \ref{bp_attracts_svm} and the KKT conditions.

Appendix \ref{app:hm_mdp_linprog} gives an alternate characterization of the event $P \cap C \neq \emptyset$ through a linear program which is of theoretical interest.

As a corollary of this theorem we can characterized when MD/MDP or SVM/MD are equivalent.
\begin{corollary}
\label{svm_md}
The hard margin SVM and MD directions are equivalent if and only if all three of hard margin SVM, MD and MDP are equivalent i.e.
$$\mathbf{w}_{hm-svm} \propto \mathbf{w}_{md} \iff \mathbf{w}_{md} \propto \mathbf{w}_{mdp}.$$
\end{corollary}

Another corollary of this theorem is that hard margin SVM is always the MDP of the support vectors.
\begin{corollary}\label{cor:svm_cropped_mdp}
Let $V$ be the set of support vectors for hard margin SVM, then $\mathbf{w}_{hm-svm}$ is the MDP of $V$.
\end{corollary}
This Corollary says that we can interpret hard margin SVM as a cropped MDP (i.e. it ignores points which are far away from the separating hyperplane).

\subsection{Hard Margin KKT Conditions}\label{ss:hm_kkt}
Derivation and discussion of the KKT conditions can be found in \cite{mohri2012foundations}. From the Lagrangian of Problem (\ref{eq:svm_hard}) we can derive the KKT conditions
\begin{align}
\mathbf{w}_{hm-svm} = \sum_{i=1}^n \alpha_i y_i \mathbf{x}_i, \label{eq:kkt_hm1}  \\ 
\sum_{i=1}^n \alpha_i y_i = 0, \label{eq:kkt_hm2}\\ 
\alpha_i = 0 \text{ or } y_i(\mathbf{w}  \cdot \mathbf{x}_i + b) = 1, \label{eq:kkt_hm3}
\end{align}

with $\alpha_i \ge 0$ for each $i = 1, \dots, n$.  

Condition (\ref{eq:kkt_hm2}) says that the sum of the weights in both classes has to be equal. Combining this with (\ref{eq:kkt_hm1})  we find that the hard margin SVM direction is given by
\begin{equation}
\label{eq:kkt_hm_cnvx}
\mathbf{w}_{hm-svm} \propto \sum_{i \in I_+} \frac{\alpha_i}{A}  \mathbf{x}_i - \sum_{i \in I_-} \frac{\alpha_i}{A} \mathbf{x}_i,
\end{equation}
where $\sum_{i \in I_+}\alpha_i = \sum_{i \in I_-}\alpha_i := A$. Thus $\mathbf{w}_{hm-svm}  \in C$ i.e. the hard margin SVM direction is always a convex direction. As discussed in \cite{bennett2000duality,phamthesis} hard margin SVM is equivalent to finding the nearest points in the convex hulls of the two classes.

The last KKT condition (\ref{eq:kkt_hm3})  says that a point $ \mathbf{x}_i $ either lies on one of the marginal hyperplanes $\{ \mathbf{x}  | \mathbf{w}_{hm-svm} ^T \mathbf{x} = \pm 1\}$ or receives zero weight. In the former case when $\alpha_i \neq 0$, $\mathbf{x}_i$ is called a \textit{support vector}. 

The margin $\rho$ is defined as the minimum distance from a training point to the separating hyperplane; $\rho$ is also the orthogonal distance from the marginal hyperplanes to the separating hyperplane. The margin width is given by the magnitude of the normal vector
\begin{equation}
\rho^2 = \frac{1}{||\mathbf{w}_{hm-svm} ||_2^2}  = \frac{1}{\sum_{i=1}^n \alpha_i}:= \frac{1}{||\mathbf{\alpha}||_1}.
\end{equation}

\subsection{Proofs for Hard Margin SVM}

The following lemma about SVM and MDP is a consequence of the fact that complete data piling directions satisfy the SVM KKT conditions.
\begin{lemma}
\label{svm_data_piling}
If hard margin SVM has complete data piling then the SVM direction is equivalent to the MDP direction i.e.
$$\mathbf{w}_{hm-svm} \in P \implies \mathbf{w}_{hm-svm} \propto \mathbf{w}_{mdp}.$$
\end{lemma}


\begin{lemma}
\label{bp_attracts_svm}
If $P \cap C \neq \emptyset$ then $\mathbf{w}_{svm} \in P \cap C$.
\end{lemma}

\begin{proof}
Let $\mathbf{v} \in P \cap C$. We show $\mathbf{v}$ satisfies the KKT conditions. The lemma then follows since the KKT conditions necessary and sufficient for hard margin SVM (the constraints are qualified, see Chapter 4 of \citealt{mohri2012foundations}).

Since $\mathbf{v} \in C$ we have that $\mathbf{v} \propto \mathbf{c}_+ - \mathbf{c}_-$ where $\mathbf{c}_j \in  \text{conv}(\{\mathbf{x}_i \}_{i \in I_j})$. For some constant $a >0$ 
$$\mathbf{v} = a \left(\sum_{i \in I_+} \lambda_i \mathbf{x}_i - \sum_{i \in I_-} \lambda_i \mathbf{x}_i \right),$$
where 
$$ \sum_{i \in I_+} \lambda_i = \sum_{i \in I_+} \lambda_i = 1 \text{ and } \lambda_i \ge 0.$$
Since $\mathbf{v} \in P$ we can select $b, \mathbf{v}$ such that
$$ y_i(\mathbf{x}_i \cdot \mathbf{v} + b) = 1 \text{ } \forall i.$$
But these three equations are the KKT conditions with $\alpha_i = a \lambda_i$.

\end{proof}


\section{Soft Margin SVM Small and Large $C$ Regimes}\label{s:softsvm}

This section characterizes the behavior of SVM for the small and large regimes of the cost parameter $C$. We make no assumptions about the dimension of the data $d$. We state the main results for the small and large $C$ regimes, provide the KKT conditions, then prove the tuning regimes results. 

We first make two geometric definitions that play an important role in characterizing SVM's tuning behavior. The two class \textit{diameter} measures the spread of the data.

\begin{definition}
\label{def:diam}
Let the two class diameter be
$$D := \max_{\mathbf{x}_+ \in I_+, \mathbf{x}_- \in I_-} ||\mathbf{x}_+ - \mathbf{x}_-||.$$
\end{definition}
The \textit{gap} measures the separation between the two data classes.

\begin{definition}
\label{def:gap}
Let the two class gap $G$ be the minimum distance between points in the convex hulls of the two classes i.e.
$$G := \min_{\mathbf{c}_j \in \text{conv}(\{\mathbf{x}_i \}_{i \in I_j})} ||\mathbf{c}_+ - \mathbf{c}_-||.$$
\end{definition}
If the data are not linearly separable then $G = 0$. 

Using the above geometric quantities we define two threshold values of $C$ which determine when the SVM enters its different behavior regimes. 
\begin{definition}
\label{def:c_small}
For two classes of data let
\begin{equation}
C_{\text{small}} := \frac{2}{\max{(n_+, n_-)} D^2},
\end{equation}
where $D$ is the diameter of the training data. 
\end{definition}

\begin{definition}
\label{def:c_large}
If the two data classes are linearly separable let
\begin{equation}
C_{\text{large}} := \frac{2}{G^2},
\end{equation}
 where $G$ is the gap between the classes.
\end{definition}

As illustrated in Figures \ref{fig:svm_bal} and \ref{fig:svm_unbal}, the main result for the small $C$ regime is given by Theorem \ref{thm:soft_small_C} and Corollary \ref{cor:small_c_explosion}. We call the support vectors lying strictly within the margin \textit{slack vectors} (Definition \ref{def:slack_vec}).
\begin{theorem}
\label{thm:soft_small_C}
When every point in the smaller (negative) class is a slack vector,
\begin{itemize}
\item if the classes are balanced then the SVM direction becomes the mean difference direction i.e. $\mathbf{w}_{svm} \propto \mathbf{w}_{md}$.

\item if the classes are unbalanced then the SVM direction satisfies the constraints in Equations \ref{eq:cropped_md},  \ref{eq:cropped_md_constraint} making it a cropped mean difference.

\begin{equation}
\label{eq:cropped_md}
\mathbf{w}_{svm} = \sum_{i \in M_+} \alpha_i \mathbf{x}_i  + C\sum_{i \in L_+} \mathbf{x}_i - C \sum_{i \in I_-} \mathbf{x}_i,
\end{equation}
subject to 
\begin{equation}
\label{eq:cropped_md_constraint}
\sum_{i \in M_+} \alpha_i = C(|L_+| - n_- ).
\end{equation}

\end{itemize}
Furthermore, $C < C_{\text{small}}$ is a sufficient condition such that every point in the smaller class is a slack vector.
\end{theorem}

Theorem \ref{thm:soft_small_C} characterizes a kind of cropped mean difference. The mean difference direction points between the mean of the first class and the mean of the second class. Recall $\mathbf{w}_{svm}$ always goes between points in the convex hulls of the two classes. Equation \ref{eq:cropped_md} says that in the small $C$ regime $\mathbf{w}_{svm} $ points between the mean of the smaller (negative) class (the third term) and a point that is close to the mean in the larger (positive) class. The cropping happens by ignoring non-support vectors. While points on the margin do not necessarily receive equal weight, Equation \ref{eq:cropped_md_constraint} bounds the amount of weight put on points on margin points.  Note Equations \ref{eq:cropped_md}, \ref{eq:cropped_md_constraint}  are stronger than the simple constraint that $\sum_{i \in I_+} \alpha_i = n_- C$  (Lemma 2 from \cite{hastie2004entire}) since  all of the slack vectors in the positive class receive the same weight.  

Lemma \ref{lem:margin_dim_bound} strengthens Lemma \ref{thm:soft_small_C} in the case $n_+ >> d$ (i.e. there can't be too many margin vectors in Equation \ref{eq:cropped_md} )
\begin{lemma}
\label{lem:margin_dim_bound}
If the data are in general position the larger class can have at most $n_- +d - 1$ support vectors.
\end{lemma}

As $C$ continues to shrink past $C_{\text{small}}$ the margin width continues to grow. Eventually the separating hyperplane will be pushed past the smaller class and every training point will be classified to the larger class (see Figure \ref{fig:unbal_error}). Note this results follows from the proofs in Section \ref{ss:small_c_proofs}. 
\begin{corollary}
\label{cor:small_c_explosion}
If the classes are unbalanced and $C < \frac{1}{2}C_{\text{small}}$ then every training point is classified to the larger (positive) class.
\end{corollary}

If the data are separable then in the large $C$ regime soft margin SVM becomes equivalent to hard margin SVM for sufficiently large $C$.
\begin{theorem}
\label{thm:soft_large_C}
If the training data are separable then when $C > C_{\text{large}}$, soft margin SVM is equivalent to the hard margin SVM solution i.e. $\mathbf{w}_{svm} = \mathbf{w}_{hm-svm} $.
\end{theorem}

Note that $C_{\text{small}}$ and $C_{\text{large}}$ are lower and upper bounds---their respective limiting behavior may happen for $C$ larger that $C_{\text{small}}$ and $C$ smaller than $C_{\text{large}}$. In practice, these threshold values are a reasonable approximation. Furthermore, the $\frac{1}{D^2}$ scales is important for small values of $C$  (this can be seen in the proofs of Corollary \ref{margin_explode} and Lemma \ref{lem:totalslack})

\subsection{Soft Margin SVM KKT Conditions}\label{ss:soft_kkt}
The KKT conditions for soft margin SVM are (see \citealt{mohri2012foundations} for derivations)

\begin{equation}
\label{eq:kkt_sm1}
\mathbf{w}_{svm} = \sum_{i \in I_+} \alpha_i \mathbf{x}_i - \sum_{i \in I_-} \alpha_i \mathbf{x}_i,
\end{equation}

\begin{equation}
\label{eq:kkt_sm2}
\sum_{i \in I_+} \alpha_i = \sum_{i \in I_-} \alpha_i := A,
\end{equation}

\begin{equation}
\label{eq:kkt_sm3}
\alpha_i + \mu_i = C \text{ for } i= 1, \dots, n,
\end{equation}

\begin{equation}
\label{eq:kkt_sm4}
\alpha_i = 0 \text{ or } y_i(\mathbf{w} \cdot \mathbf{x}_i + b) = 1 - \xi_i \text{ for } i = 1, \dots, n,
\end{equation}

\begin{equation}
\label{eq:kkt_sm5}
\xi_i = 0 \text{ or } \mu_i = 0 \text{ for each } i,
\end{equation}

For soft margin SVM we define the marginal hyper planes to be $\{\mathbf{x} | \mathbf{x}^T \mathbf{w}_{svm} = \pm 1\}$ and the margin width (or just margin), $\rho$ the  distance from the separating hyperplane to the marginal hyperplanes. By construction $\rho = \frac{1}{||\mathbf{w}_{svm}||}$. For soft margin SVM, the margin does not have the same meaning as in the hard margin case, but still plays an important role. In particular, a points is a support vector if and only if it is contained within the marginal hyperplanes.

As with hard margin SVM, the soft margin direction is always a convex direction. Again points $\mathbf{x}_i$ such that $\alpha_i \neq 0$ are called support vectors. We further separate support vectors into two types.

\begin{definition}\label{def:margin_vec}
Margin vectors are support vectors $\mathbf{x}_i$ such $\alpha_i \neq 0$ and  $\xi_i = 0$.
\end{definition}

\begin{definition} \label{def:slack_vec}
Slack vectors are support vectors $\mathbf{x}_i$ such $\alpha_i \neq 0$ and  $\xi_i > 0$.
\end{definition}

Margin vectors are support vectors lying on one of the two marginal hyperplanes. Slack vectors are support vectors lying strictly on the inside of the marginal hyperplanes. Call the set of margin vectors in each class $M_j$ and the set of slack vectors $L_j$ for $j= \pm$. 

The KKT conditions imply
\begin{itemize}
\item all support vectors receive weight upper bounded by $C$  ($\mathbf{x}_i \in M_j \implies 0 < \alpha_i \le C$)
\item slack vectors receive weight exactly $C$ ($\mathbf{x}_i \in L_j \implies \alpha_i = C$)
\end{itemize}

Furthermore, the following constraint balances the weights between the two classes
\begin{equation}
\label{eq:sm_weight_balance}
C |L_+|  + \sum_{i \in M_+} \alpha_i = C |L_-|  + \sum_{i \in M_-} \alpha_i.
\end{equation}
We assume that the positive class is the larger of the two classes i.e. $n_+ \ge n_-$. Unbalanced classes means $n_+ > n_-$.

\subsection{Proofs for Small $C$ Regime} \label{ss:small_c_proofs}

As $C \to 0$ the margin width increases to infinity ($\rho \to \infty$). As the margin width grows as many points as possible become slack vectors and all slack vectors get the same weight $\alpha_i = C$. Hence if the classes are balanced the SVM direction will be equivalent to the mean difference. If the classes are unbalanced then there will be some margin vectors which receive weight $ \alpha_i \le C$. The number of margin vectors is bounded by the class sizes and the dimension. 

Note the diameter, $D$, does not change if we consider the convex hull of the two classes (proof of Lemma \ref{diam} is a straightforward exercise).

\begin{lemma}
\label{diam}
$$ \max_{\mathbf{c}_j \in \text{conv}(\{\mathbf{x}_i \}_{i \in I_j})}||\mathbf{c}_+ - \mathbf{c}_-|| = \max_{\mathbf{x}_j \in I_+} ||\mathbf{x}_+ - \mathbf{x}_-|| =: D.$$
\end{lemma}
As $C \to 0$ the magnitude of $\mathbf{w}_{svm}$ goes to zero. In particular, the KKT conditions give the following bound.
\begin{lemma}
For a given $C$ the magnitude of the SVM solution is
$$||\mathbf{w}_{svm}|| \le n_+ C \cdot D.$$
\end{lemma}

\begin{proof}
From the KKT conditions we have
$$\mathbf{w}_{svm} = \sum_{i \in I_+ } \alpha_i \mathbf{x}_i - \sum_{i \in I_- } \alpha_i \mathbf{x}_i$$
and 
$$\sum_{i \in I_+ } \alpha_i = \sum_{i \in I_- } \alpha_i =: A.$$
Computing the magnitude of $\mathbf{w}_{svm}$

$$||\mathbf{w}_{svm} || = A \left| \left| \sum_{i \in I_+ } \frac{\alpha_i}{A} \mathbf{x}_i  -  \sum_{i \in I_- } \frac{\alpha_i}{A} \mathbf{x}_i  \right|\right|.$$
Since the two terms are convex combinations we get 
$$|| \mathbf{w}_{svm} ||  \le A \sup_{\mathbf{c}_j \in \text{conv}(\{\mathbf{x}_i \}_{i \in I_j})} ||\mathbf{c}_+ - \mathbf{c}_-||.$$
applying Lemma \ref{diam}
$$|| \mathbf{w}_{svm} || = A \max_{\mathbf{x}_j \in I_+} ||\mathbf{x}_+ - \mathbf{x}_-|| $$
$$|| \mathbf{w}_{svm} || = AD.$$
Since $0 \le \alpha_i \le C$ we get $A \le n_1 C$ thus proving the bound.
\end{proof}

Since the magnitude of $\mathbf{w}_{svm}$ determines the margin width, using the previous lemma we get the following corollary.
\begin{corollary}
\label{margin_explode}
The margin $\rho$ goes to infinity as $C$ goes to zero. In particular 
$$\rho = \frac{1}{||\mathbf{w}_{svm}||} \ge \frac{1}{n_+C D}.$$
\end{corollary}
Since the margin width increases, for small enough $C$ the smaller class becomes all slack variables.

\begin{lemma} \label{lem:totalslack}
If $C < C_{\text{small}}$ then all points in the smaller class become slack vectors ($\xi_i > 0$ for all $i \in I_-$).
\end{lemma}

\begin{proof}
By Corollary \ref{margin_explode} the margin width goes to infinity as $C \to 0$ since
$$\rho \ge \frac{1}{n_+C D}.$$
Recall the margin width, $\rho$, is the distance from the separating hyperplane to the marginal hyperplanes. Note that if $\rho > \frac{1}{2}D$ then at least one class must be complete slack. Thus if $C < \frac{2}{ n_1 D^2}$ at least one class must be complete slack i.e. $\xi_i > 0$ for all $i \in I_j$ for $j = +$ and/or $j=-$. If the classes are balanced then either class can become complete slack (or both classes).

If the classes are unbalanced i.e. $n_- < n_+$ then the smaller class becomes complete slack. To see this, assume for the sake of contradiction that the larger class becomes complete slack i.e. $\xi_i \neq 0$ for each $i \in I_+$. Then the KKT conditions imply $\alpha_i = C$ for each $i \in I_+$. KKT condition \ref{eq:kkt_sm2} says
$$\sum_{i \in I_+ } \alpha_i = \sum_{i \in I_-} \alpha_i $$
$$ n_+ C = \sum_{i \in I_- } \alpha_i.$$
But $\alpha_i \le C$ and $n_- < n_+$ by assumption therefore this constraint cannot be satisfied.
\end{proof}

If the classes are balanced then the margin swallows both classes and the SVM direction becomes the mean difference direction.
\begin{lemma}
If the classes are balanced and $C < C_{\text{small}}$ the SVM direction is equivalent to the mean difference direction i.e. $\mathbf{w}_{svm} \propto \mathbf{w}_{md}$.
\end{lemma}

\begin{proof}
When $C < C_{\text{small}}$ one of the classes (without loss of generality the negative class) becomes slack i.e. $\xi_i >0$ for each $i \in I_-$ thus $\alpha_i = C$ for each $i \in I_-$. The KKT conditions then require
$$\sum_{i \in I_+} \alpha_i = \sum_{i \in I_-} \alpha_i  = n_- C.$$
Since $\alpha_i \le C$ and $|I_+| = n_-$ this constraint can only be satisfied if $\alpha_i = C$ for each $i \in I_+$. We now have
$$\mathbf{w}_{svm} = \sum_{i \in I_+} C \mathbf{x}_i - \sum_{i \in I_-} C \mathbf{x}_i$$
$$ \mathbf{w}_{svm}  = C \frac{n}{2} (\bar{\mathbf{x}}_+ - \bar{\mathbf{x}}_-) \propto \mathbf{w}_{md}.$$
\end{proof}

\begin{lemma}
\label{lem:cropped_md_char}
If the classes are unbalanced and $C < C_{\text{small}}$ the SVM solution satisfies the the constraints in Equations \ref{eq:cropped_md}, \ref{eq:cropped_md_constraint}.
\end{lemma}

\begin{proof}
Recall for $C < C_{\text{small}}$ we have $\xi_i > 0$ for $i \in I_-$. From the KKT conditions $\xi_i > 0 \implies \mu_i = 0 \implies \alpha_i = 0$ meaning $\alpha_i = C$ for each $i \in I_-$. The weight balance constraint \ref{eq:sm_weight_balance} from the KKT conditions becomes
$$C |L_+|  + \sum_{i \in M_+ } \alpha_i = C |L_-|  + \sum_{i \in M_- } \alpha_i,$$
which then implies the conditions on $\mathbf{w}_{svm}$.
\end{proof}

\begin{corollary}
\label{bounce}
When $C < C_{\text{small}}$ the larger (positive) class can have at most $n_-$ slack vectors. If the larger class has more than $n_-$ support vectors then at least one of them must be a margin vector.
\end{corollary}

\subsection{Proofs for Large $C$ Regime}

\begin{lemma}
\label{lem:gap_ineq}
If there is at least one slack vector then for a given $C$
$$||\mathbf{w}_{svm}|| \ge C G,$$
or equivalently
$$\rho \le \frac{1}{C G},$$
where $G$ is the class gap.
\end{lemma}

\begin{proof}
From the KKT conditions
$$||\mathbf{w}_{svm}|| = || \sum_{i \in I_+ } \alpha_i \mathbf{x}_i - \sum_{i \in I_- } \alpha_i \mathbf{x}_i ||,$$
$$ ||\mathbf{w}_{svm}|| = A || \sum_{i \in I_+ } \frac{\alpha_i}{A} \mathbf{x}_i - \sum_{i \in I_- } \frac{\alpha_i}{A} \mathbf{x}_i||,$$
where $A = \sum_{i \in I_+ } \alpha_i = \sum_{i \in I_-} \alpha_i$. Since the two sums are convex combinations, using the definition of $G$ we get
$$||\mathbf{w}_{svm}||  \ge A G.$$
Since there is at least one slack vector there is at least one $i$ such that $\alpha_i = C$ thus $A \ge C$ and the result follows. 
\end{proof}

\section{Summary of SVM Regimes} \label{s:regimes}

For sufficiently small values of $C$, SVM is related to the mean difference. When the data are separable, for sufficiently large values of $C$ soft margin SVM is equivalent to hard margin SVM. We note this discussion applies more broadly than just binary, linear SVM. For example, when a kernel is used, SVM becomes related to the kernel mean difference classifier. Often multi-class classification problems are reduced to a number of binary class problems e.g. using \textit{one vs. one} (OVO) or \textit{one vs. all} (OVA) schemes. Our results apply to each of these binary classification problems. For example, in a multi-class problem, even if the classes are roughly balanced, the OVA scheme may produce unbalanced classes where the behavior discussed in Section \ref{ss:class_imbal} becomes applicable. 

\subsection{Small $C$ Regime and the Mean Difference} \label{ss:small_c_md}

For sufficiently small $C$ (when every point in the smaller class is a slack vector) Theorem \ref{thm:soft_small_C} shows how soft margin SVM is related to the mean difference. 

If the data are unbalanced then the SVM direction becomes a cropped mean difference direction as characterized by Equations \ref{eq:cropped_md},  \ref{eq:cropped_md_constraint}. The direction points from the mean of the smaller class to a cropped mean of a subset of points in the larger class. The cropped mean of the larger class gives equal weight to slack vectors, puts smaller weight on margin vectors and ignores points that are outside the margin (non-support vectors). Furthermore, the number of margin vectors is bounded by the dimension when the data are in general position (Lemma \ref{lem:margin_dim_bound}).

In the small $C$ regime, if the data are balanced then the SVM direction becomes exactly the mean difference direction. Note Lemma 1 from \cite{hastie2004entire} proves this result for balanced classes,  proves a weaker version in the unbalanced case, does not give the threshold $C_{\text{small}}$, and does not discuss the connection between SVM and the MD classifier.

The lower bound $C_{\text{small}}$ is important because it shows SVM's MD like behavior applies for every dataset set. Furthermore, it shows that the value of $C$ where the MD like behavior begins depends on the data diameter and class sizes $\left(\text{i.e. is proportional to } \frac{1}{\max{(n_+, n_-)} D^2}\right)$. This dependence on the data diameter has important consequences for cross-validation which are discussed in Section \ref{s:applications}.

Note the cropped MD interpretation is often valid for a wide range of $C$ (i.e. values of $C$ larger than $C_{\text{small}}$). In particular, as $C$ shrinks, more vectors become slack vectors receiving equal weight (see  proofs and results in Section \ref{s:softsvm}). As $C$ shrinks to $C_{\text{small}}$, the angle between SVM and the cropped MD defined in Theorem \ref{thm:soft_small_C} approaches zero. This can be seen, for example, in Figure \ref{fig:unbal_angle}.

Finally, note that the relation between SVM and the MD also relates SVM to a larger set of classifiers by taking data transformation into account (see Section \ref{ss:data_transform}). It is common to apply a transformation to the data before fitting SVM (e.g. mean centering then scaling by some covariance matrix estimate). In this case, the small $C$ regime of SVM will be a (cropped) version of the transformed MD classifier. This insight connects SVM to, for example, the naive Bayes classifier. Similarly, our results also connect kernel SVM to the kernel (cropped) MD classifier.

SVM's MD behavior discussed in this section raises the question of how much performance gain SVM achieves over (robust, transformed) mean difference classifiers. This is discussed more in Section \ref{ss:dis_svm_other_classifiers}.

\subsection{Class Imbalance and the MD Regime} \label{ss:class_imbal}
Theorem \ref{thm:soft_small_C} gives some insights into SVM when the classes are imbalanced. When SVM is in the MD regime as discussed above (i.e. $C \le C_{\text{small}}$), every point in the smaller (negative) class has to be a support vector receiving equal weight. In some scenarios the MD or a cropped MD may perform very well. However, this result says in the small $C$ regime, SVM cannot crop the smaller class (it can still crop the smaller class when $C > C_{\text{small}}$). This insight can explain some scenarios where SVM performs well for small values of $C$, but then its performance suddenly degrades for even smaller values of $C$ (i.e. an outlier is forced into the smaller class's slack vectors).

Lemma \ref{lem:margin_dim_bound} says that (under weak conditions) the larger (positive) class can have at most $n_- + d + 1$ support vectors ($n_-$ = size of the smaller class). In the case $n_+ >> n_-,d$ then SVM can only use a small number of data points from the larger class to estimate the SVM direction (this is true for all values of C). This means SVM is forced to do a lot of cropping for the larger (positive) class which may be a good thing in some scenarios (i.e. if the larger class has many outliers). 

\subsection{Small $C$ Regime and Margin Bounce} \label{ss:margin_bounce}

As $C$ shrinks, the margin (distance between the marginal hyperplanes) increases. When the classes are unbalanced, the marginal hyperplane of the larger class has to stay within the convex hull of the larger class causing the separating hyperplane to move off to infinity. For small enough values of $C$ ($\le \frac{1}{2} C_{\text{small}}$), this means the separating hyperplane is pushed past the smaller class and every point is classified to the larger class (Corollary \ref{cor:small_c_explosion}). We call this behavior \textit{margin bounce} (see Figure \ref{fig:unbal_large_C} for an example).  In other words, for small values of $C$, SVM picks a reasonable direction, but a bad intercept.

When the classes are exactly balanced, the margin bounc may or may not happen (we have seen data examples of both). It would be an interesting follow up question to determine conditions for when the margin bounce happens for balanced classes. 

This insight has a few consequences. 
\begin{enumerate}
\item For Figure \ref{fig:unbal_error} (unbalanced classes) it explains why the three tuning error curves are large for small values of $C$.
\item For Figure  \ref{fig:bal_error} (balanced classes) it explains why only the cross-validation error curve is bad for small values of $C$, but the tuning and test set error curves are fine (i.e. the cross-validation training sets are typically unbalanced).
\item For small values of $C$ SVM picks a bad intercept, but a fine direction. We exploit this fact in Section \ref{ss:better_svm_intercept} to develop an improved intercept for SVM
\item The value of $C$ when the margin starts exploding depends on the diameter of the two classes. This has important implications for cross-validation which are discussed in Section \ref{ss:cv_tuning_insights} 
\end{enumerate}

\subsection{Large $C$ Regime and the Hard-Margin SVM}
If the data are separable, Theorem \ref{thm:soft_large_C} says that for sufficiently large values of $C$, soft margin SVM will be equivalent to hard margin SVM. Note that in high-dimensions (i.e. $d > n$) the data are always separable. If the original dataset is non-separable, but a kernel is used the transformed dataset may in fact be separable (for example, if the implicit kernel dimension is larger than $n$).

Furthermore, the value of $C$ above which soft-margin SVM becomes equivalent to hard margin SVM depends on the gap between the two classes (see Definition \ref{def:gap}). This can have important consequences for cross-validation as discussed in Section \ref{ss:cv_tuning_insights}.

\subsection{Hard-Margin SVM and the (cropped) Maximal Data Piling Direction}

In high dimensions, (i.e. $d \ge n -1$)  Theorem \ref{thm:hm_mdp} gives geometric conditions for when hard margin SVM gives complete data piling i.e. when the SVM direction is equivalent to the MDP direction. Hard margin SVM always has some data piling; support vectors in the same class project to the same point. In this case SVM is the MDP direction of the support vectors. In this sense, hard margin SVM can be viewed as a cropped MDP direction where points away from the margin are ignored.

Complete data piling is a strict constraint and the SVM normal vector can usually wiggle away from the MDP direction to find a larger margin.  This raises the question: is complete data piling with hard margin SVM a probability zero event when the data are generated by an absolutely continuous distribution? We suspect the answer is no: it occurs with positive, but typically small probability. For example consider three points in $\mathbb{R}^2$.

Often data piling may not be desirable e.g. the normal vector may be sensitive to small scale noise artifacts \cite{marron2007distance}. Additionally, the projected data have a degenerate distribution since multiple data points lie on top of each other. However there are cases, such as an autocorrelated noise distribution, when the maximal data piling direction performs well, \cite{miao2015class}.

Corollary \ref{cor:svm_cropped_mdp} (SVM is the MDP of the support vectors) also gives an alternative characterization of hard margin SVM. Hard margin SVM searches over every subset of the data points which have a nonempty set of complete data piling directions, computes the MDP of each such subset, and selects the direction giving the largest separation. This characterization is mathematically interesting because it says we can \textit{a priori} restrict the hard margin SVM optimization problem, Equation \ref{eq:svm_hard}, to search over a finite set of directions (i.e. the complete data piling directions of the subsets of the data). Furthermore, in some cases, the MDP (Equation \ref{eq:mdp}) can be cheaply computed or approximated. For example, the analyst may use a low rank approximation to $\widehat{\Sigma}^-$ and/or select a judicious subset of data points. In these scenarios, it may make sense to approximate hard margin SVM with the MDP.

\section{Applications of SVM Regimes} \label{s:applications}

There are are number of ways of tuning soft margin SVM including: heuristic choice, random search, Nelder-Mead and cross-validation (\citealt{nelder1965simplex, mattera1999support, chapelle2000model, hsu2003practical, christmann2005determination, steinwart2008support}). In practice one of the most popular methods is to select $C$ which optimizes the $K$-fold cross-validation error (\citealt{friedman2001elements, hsu2003practical}). Note for very unbalanced classes, the cross-validation error metric can be replaced with other test set error metrics such as F-score, Kappa, precision/recall, balanced error, AUC \cite{tan2005introduction}. This section focuses on test set error, but the discussion is relevant to these other error metrics. The discussion also focuses on cross-validation, but similar conclusions can be drawn when a fixed validation set is used. Furthermore, these insights also apply to using cross-validation to estimate the true test set error. 

\subsection{Tuning SVM via Cross-Validation} \label{ss:cv_tuning_insights}

Tuning SVM using cross-validation means attempting to estimate the tuning curve of the test set (the green line marked with triangles in Figures \ref{fig:bal_error}, \ref{fig:unbal_error}) using the tuning curve from cross-validation (the red line marked with circles). It is known that the optimal hyper-parameter settings for the full training set (of size $n$) may differ from the optimal settings for the cross-validation sets (of size $\left(1 - \frac{1}{k} \right)n$); for example, the smaller dataset often favors larger values of $C$ (more regularization) \cite{steinwart2008support}.

The results of this paper give a number of insights into how features of the data cause the cross-validation tuning curve to differ from the test set tuning curve. In particular, we have show that the tuning curve is sensitive to
\begin{enumerate} 
\item balanced vs. unbalanced classes, 
\item the two class diameter $D$, 
\item whether or not the classes are separable,
\item whether or not $ d \ge n -1$,
\item the gap between the two classes $G$.
\end{enumerate}

Each of these characteristics can change between the full training set and the cross-validation training sets. When the characteristics change, so can SVM's behavior for small and large values of $C$. Therefore SVM may behave differently for the cross-validation folds than for the full training data. 

One dramatic example of this change in behavior can be seen in Figure \ref{fig:bal_error} as discussed in Sections \ref{ss:margin_bounce}, and \ref{ss:motivation}. In this case, the full dataset is balanced, but the cross-validation folds are typically unbalanced.

Another example of tuning behavior differences between the training and cross-validation data can be seen by looking carefully at Figure \ref{fig:unbal_error}. In this figure we can see the cross-validation error rate shoots up for larger values of $C$ than the train/test error rates. The error increases dramatically for small values of $C$ because of the margin bounce phenomena discussed in Section \ref{ss:margin_bounce}. The value of $C_{\text{small}}$ that guarantees this behavior is a function of the two class diameter $D$ (see Definition \ref{def:c_small}). Since there are fewer points in the cross-validation training set, the diameter is smaller meaning the value of $C_{\text{small}}$ is larger causing the margin to explode for larger values of $C$. 

Different data domains in terms of $n << d, n \sim d$, and $n >> d$ can make the above characteristics more or less sensitive to change induced by subsampling. For example, if $n >> d$ then subsampling is least likely to change whether $d \ge n -1$ or significantly modify the diameter $D$. With a kernel, however, even if the original $n >> d$ then it may no longer be true that $n >> d_{\text{implicit}}$ where $d_{\text{implicit}}$ is the dimension of the implicit kernel space. An interesting, possible exception to this was given by \cite{rahimi2008random} where $d_{\text{implicit}}$ may be small.

When $n$ is larger than $d$, but not by much, then subsampling is likely to change whether or not $d \ge n -1$ and whether or not the data are separable. In this case the full training data may not be separable, but the cross-validation sets may be. This means large values of $C$ will cause soft margin SVM to become hard margin SVM for cross-validation, but never for the full training data. This could result in the SVM direction being very different between cross-validation and training. 

When $d \ge n -1$ soft margin SVM will become hard margin SVM for $C \ge C_{\text{large}}$ which depends on the gap $G$ between the two classes. Subsampling the data will cause this gap to increase meaning $C_{\text{large}}$ decreases. In this case the hard margin behavior will occur for smaller values of $C$ in the cross-validation sets than for the full training set. 

It is desirable to perform cross-validation in a way that is least likely to change some of the above characteristics between the full and the cross-validation training data set. For example, 
\begin{itemize}
\item If the full training data are balanced one should ensure the cross-validation training classes are also balanced.
\item Cross-validation with a large number of folds (e.g. leave one out CV) is least likely to modify the above characteristics of the data. 
\item When $n > d$ it could be judicious to make sure that $n_{cv} > d$ for each cross-validation training set. 
\item \cite{chapelle2000model} (Section 4) suggests re-scaling the data using the covariance matrix. The analyst may modify this idea by additionally rescaling each cross-validation training set such that the diameter is (approximately) the same as the diameter of the full training set.
\item Previous papers have proposed default values for $C$ based on the given dataset \cite{mattera1999support, cherkassky2004practical}. Our results suggest other default values in the interval $[C_{\text{small}}, C_{\text{large}}]$ (when the latter exists) may be reasonable. Furthermore, default values which lie in the middle of this range may be preferable. For example, the analyst may try a simple MD classifier (producing similar results to a small $C$), one moderate and one large  value of $C$ for SVM.
\end{itemize}

\subsection{Improved SVM Intercept for Cross-Validation} \label{ss:better_svm_intercept}

As discussed in Section \ref{ss:margin_bounce}, SVM's intercept can be problematic for small values of $C$; for small values of $C$ the margin bounce causes every point to be classified to the larger of the two classes. This fact alone may not be concerning, however, as Theorem \ref{thm:soft_small_C} and Definition \ref{def:c_small} show, SVM can behave differently, as a function of $C$, for cross-validation and on the full data set.  The subsampled data sets for cross-validation will have a smaller diameter, $D$, meaning the threshold $C_{small}$ is larger for these datasets than for the full dataset. In particular, the margin explosion happens a larger value of $C$ during cross-validation than it does for the full dataset. This will cause the cross-validation test set error to be large for values of $C$ where the test set error may in fact be small.

We can fix this issue by modifying the SVM intercept as follows. Note that previous papers have suggesting modifying SVM's intercept \cite{crisp2000geometric}. Suppose we fit SVM to a dataset and it returns normal vector and intercept $\mathbf{w}_{svm}$ and $b_{svm}$ respectively. Furthermore, define the \textit{SVM centroids} by
$$
\mathbf{m}_{svm, +} = \frac{1}{A}\sum_{i \in I_+} \alpha_i \mathbf{x_i},
$$
where the $\alpha_i$ are the support vectors weights and $A$ is the total weight (Equation \ref{eq:kkt_sm2}). Note this is a convex combination of points in the positive class (hence the name SVM centroid). We define $\mathbf{m}_{svm, -}$ similarly for the negative class.

Next define an new intercept by
\begin{equation}
b_{centroid} := \frac{1}{2}\mathbf{w}_{svm}^T (\mathbf{m}_{svm, +} + \mathbf{m}_{svm, -})
\end{equation}
Note $b_{centroid} $ is the value such that SVM's separating hyperplane sits halfway between $\mathbf{m}_{svm, +} $ and $\mathbf{m}_{svm, -}$. Furthermore, note this quantity can be computed when a kernel is used.

The SVM intercept is only a problem when $C$ is small and one class is entirely support vectors (i.e. $\alpha_i >0 \forall i \in I_+$ or $\forall i \in I_-$). Finally, we define a new intercept as follows
\begin{equation}
b =
\begin{cases}
b_{centroid}, & \text{if one class is entirely support vectors} \\
b_{svm}, & \text{otherwise} 
\end{cases}
\end{equation}

Note that when the optimal value of $C$ is large, the margin explosion discussed in this section is not an issue and $b$ defined above will give the same result as the original $b_{svm}$.

The intercepts $b_{centroid}$ and $b$ defined above are not the only options. One could, for example, replace the SVM centroids with the class means (i.e. replace $\mathbf{m}_{svm, -}$ with $\mathbf{x}_+$ ). Alternatively, one could use cross-validation to select $b$ separately from $\mathbf{w}$. We focus on $b_{centroid}$ because it is simple can be interpreted as viewing SVM as a nearest centroid (as discussed in Section \ref{ss:md}).

Below we demonstrate an example where $b$ defined above improves SVM's test set performance. In this example, there are $n_+ = 51$ and $n_-=50$ points in each class living in $d = 100$ dimensions. The two classes are generated from Gaussians with identity covariance and means which differ only in the first coordinate; the mean of the positive class is the first standard basis vector and the mean of the negative class is negative the first standard basis vector. Note that MD is the Bayes rule in this example.  We tune SVM using using 5-fold cross-validation to select the optimal value of $C$ the compute the resulting test set error for an independent test set of 2000 points.

\begin{figure}
\centering
\begin{subfigure}[b]{0.4\textwidth}
\includegraphics[width=\textwidth]{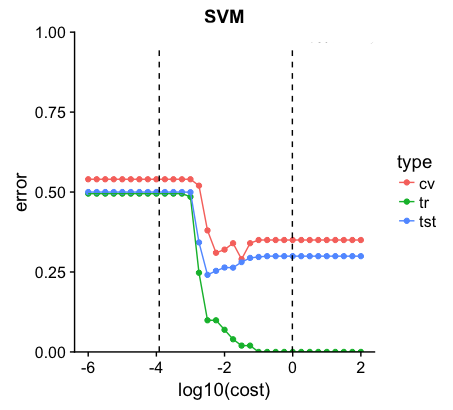}
\caption{Regular SVM intercept}
\label{fig:svm_intercept_tuning_curve_ex}
\end{subfigure}
~ 
\begin{subfigure}[b]{0.4\textwidth}
\includegraphics[width=\textwidth]{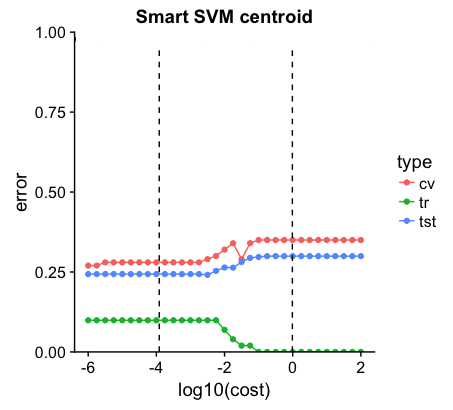}
\caption{Intercept using SVM centroids}
\label{fig:centroid_intercept_tuning_curve_ex}
\end{subfigure}
\caption{Tuning error curves for standard SVM intercept vs. improved SVM intercept.}
\label{fig:svm_interecept_ex}
\end{figure}

Figure \ref{fig:svm_interecept_ex} shows the error tuning curves (as in Figure \ref{fig:bal_error}) for the two choices of SVM intercepts for a single draw of the data. The x-axis is the tuning parameter and the y-axis is the resulting SVM error for training, testing, and 5-fold cross-validation test set error. In the left panel we see each error curve jumps up to around $50\%$ for small values of $C$ for the regular SVM intercept. Furthermore, this error explosion happens for a smaller value of $C$ for the test set error than for the cross-validation error (i.e. the blue test curve is to the left of the red cross validation curve). In the right panel, with the SVM centroid intercept, the error rate does not explode; moreover, the test error curve behaves similarly to the cross-validation curve. The curves on the right and left panels are identical for $C > 10^{-2}$.  For this data set, 5-fold cross-validation  gives a test set error of $28.1\%$ for the regular SVM intercept, but $24.35 \%$ for the SVM centroid intercept.

Over 200 repetitions of this simulation, regular SVM has an mean test set error of  $25.95\%$ (MD gives $23.95 \%$). If we replace the regular SVM intercept, $b_{svm}$ with $b$ defined above we get an average test set error of $24.80 \%$; this intercept gives an average improvement of $1.15 \%$ for this dataset (this difference is statistically significant using a paired t-test which gives a p-value of $2 \times 10^{-16}$). 

When the classes are very unbalanced other error metrics are used (e.g. F-score, AUC, Choen's Kappa, etc). If AUC is used i.e. the intercept is tuned independently of the direction, issues with the intercept discussed in this section will not occur. However, when other metrics are used the improved intercepts will likely be more effective.

The intercept $b$ defined above will not improve SVM's performance in all scenarios, but is not likely to harm the performance. The intercept $b$, however, is simple to implement and can give a better test set error.

\section{Discussion}\label{s:diss}

\subsection{Geometry of Complete Data Piling}\label{ss:data_piling_geometry_dis}

Theorems  \ref{thm:affine_piling_directions} and \ref{thm:hm_mdp} give further insight into the geometry of complete data piling directions. In this section we consider directions to be points on the unit sphere; the equivalence class of a single direction is represented by two antipodal points.

When $d \ge n$ there are an infinite number of directions $P$ that give complete data piling. If we restrict ourselves to the $n$ dimensional subspace generated by the data there are still an infinite number of directions that give complete data piling  \cite{ahn2010mdp}; within this subspace $P$ forms a great circle of directions.  Theorem \ref{thm:affine_piling_directions} says that if we further restrict ourselves to the $n -1$ dimensional affine hull of the data there is only a single direction of complete data piling and this direction is the maximal data piling direction. The aforementioned great circle of directions intersects the subspace parallel to the affine hull of the data at two points (i.e. a single direction).

Note Equation \ref{eq:mdp} shows $\mathbf{w}_{mdp}$ is a linear combination of the data and Theorem \ref{thm:affine_piling_directions} shows furthermore that $\mathbf{w}_{mdp}$ an affine direction. Finally, Theorem \ref{thm:hm_mdp} also characterizes the stronger condition when the MDP is a convex classifier (see Section \ref{ss:md})  i.e. when the MDP direction points between the convex hulls of the two classes  ($\mathbf{w}_{mdp} \in C$).

\subsection{nu-SVM and the Reduced Convex Hull} \label{ss:nu_svm}
A number of papers look at an alternative formulation of the SVM optimization problem (so called nu-SVM). These papers give an interesting, geometric perspective that characterizes soft margin SVM in terms of hard margin SVM (see citations in Section \ref{ss:lit_review}). 

Recall the convex hull of a set of points is given by $H(\{\mathbf{x}_i \}_{i=1}^n) := \left \{ \sum_{i=1}^m \lambda_i \mathbf{x}_i | \sum_{i=1}^n \lambda_i = 1, \lambda_i \ge 0  \right \}$. Suppose we decrease the upper bound on the coefficients such that $\lambda_i \le c$ for some $c \ge 0$. Define the \textit{reduced convex hull} (RCH) as
$$R_{c}(\{\mathbf{x}_i \}_{i=1}^n) := \left\{ \sum_{i=1}^n \lambda_i \mathbf{x}_i | \sum_{i=1}^n \lambda_i = 1, \lambda_i \le c \right \}$$
Note $R_c \subseteq H$, $R_c = H \iff c = 1$ and $c = \frac{1}{n} \iff R_c = \{ \frac{1}{n} \sum_{i=1}^{n} \mathbf{x}_i \}$ (i.e. a single point). Also note that, $R_c$ is not necessarily a dilation of $H$ e.g. see Figure 5 from \cite{bennett2000duality} for an example. Furthermore, define $E_c$ to be the set of \textit{extreme points} of $R_c$ (the RCH of a finite set of points is a polytope and the extreme points are the vertices of this polytope). 

Similarly to  Definition \ref{def:cvx_dir} of the convex directions for two classes, we define the set of \textit{reduced convex directions}, $RC_c$ 
\begin{definition}
\label{def:rch}
Let $0 \le c \le \min \left(\frac{1}{n_+}, \frac{1}{n_-} \right)$ and let $RC_c$ denote the set of all vectors associated with the directions that go between the $c$ reduced convex hulls convex hulls of the two classes i.e.
\end{definition}
$$RC_c = \{ a\left( \mathbf{c}_+ - \mathbf{c}_-  \right)| a \in \mathbb{R}, a \neq 0, \text{ and } \mathbf{c}_j \in R_c(\{\mathbf{x}_i \}_{i \in I_j}), j=\pm \}.$$
Similarly, let $ERC_c$ denote the set of extreme points of $RC_c$ (where the points are marked by their respective class labels). Note that even if the convex hulls of the two classes intersect, there (usually\footnote{If, for example, the class means are identical the RCH formulation may breakdown.}) exists a $c' \ge 0$ such that the $c'$ reduced convex hulls of the two classes do not intersect.

The nu-SVM literature shows that for every $C$, there exists a $c \ge 0$ that soft margin SVM direction with tuning parameter $C$ is equivalent to the hard margin SVM direction of the extreme points of the $c-$reduce convex hull of the data ($ERC_c$) which are a subset of the convex hull of the original data.

We point this geometric insight out because it gives similar geometric insights into SVM as our paper. Furthermore, the RCH formulation connects soft margin SVM to the maximal data piling direction; in particular, soft margin SVM is the MDP of the extreme points of the RCH.

\subsection{Relations Between SVM and Other Classifiers}\label{ss:dis_svm_other_classifiers}

We have shown SVM can be exactly or approximately equivalent to the mean difference or maximal data piling direction (or possibly cropped versions of these two classifiers). When the data are balanced and $C$ is sufficiently small, SVM becomes exactly the mean difference. When the data are unbalanced, SVM becomes a cropped version of the mean difference. Hard margin SVM is always the maximal data piling direction of the support vectors meaning it can be viewed as a cropped MDP. We gave conditions for when hard margin SVM is exactly the MDP of the full dataset.

These results are mathematically interesting i.e. they give conditions when a quadratic optimization problem reduces (exactly or approximately) to a problem which has a closed form solution with a simple geometric interpretation. By carefully studying how this behavior depends on the tuning parameter we give a number of insights into tuning SVM (see Section \ref{s:applications}).

Furthermore, these insights can be directly relevant to the data analyst. For example, the analyst may learn something about the data when they encounter scenarios in which SVM is either exactly or approximately equivalent to one of these simple classifiers. In scientific applications using SVM, the data analyst may want to know more about why cross-validation selects a given tuning parameter.

Our results help both practitioners and researchers transfer intuition from the MD and MDP classifiers to SVM and vice versa. The mean difference classifier is widely used (especially if one takes the data transformation perspective from Section \ref{ss:data_transform}) and a lot is known about when it works well and doesn't (e.g. if the two classes are homoskedastic point clouds). While the MDP is an active topic of research, as discussed in \cite{miao2015class}, we understand some cases when the MDP works well and doesn't. 

Finally, the results in this paper raise the question: how much performance gain does SVM achieve over more simple classifiers? For example, for a particular application it could be the case that the mean difference plus some combination of simple data transformation, robust mean estimation, and/or kernels would achieve a very similar test set error rate as SVM. This question is important to practitioners because more simple models are often favored for reasons of interpretability, computation, robustness, etc.

An interesting follow up question for researchers is to empirically compare SVM to a variety of mean difference and maximal data piling like classifiers for a large number of datasets. We suspect that in some cases, the more simple classifiers will perform very similarly to SVM and in other cases SVM will truly beat out these more simple classifiers. Finally, we recommend that practitioners keep track of at least the MD (and possibly MDP in high dimensions) when fitting SVM.



\acks{This research was supported in part by the National Science Foundation under Grant No. 1633074.}


\newpage

\appendix
\section*{Appendix A.}
In this section we prove Theorem  \ref{thm:affine_piling_directions}. Online supplementary material including code to reproduce the figures in this paper, proofs that were omitted for brevity and simulations can be found at: \url{https://github.com/idc9/svm_geometry}.

\begin{proof} \textbf{of Theorem \ref{thm:affine_piling_directions}}

We first prove the existence and uniqueness of complete data piling directions $P$ in the affine hull of the data. We then show that this unique, affine data piling direction is in fact the direction of maximal data piling.

Recall we assume that $d \ge n -1$ and the data are in general position. Let the set of affine directions $A$ be given as follows
$$A = \{  \mathbf{a}_1 - \mathbf{a}_2  |  \mathbf{a}_j \in \text{aff}(\{\mathbf{x}_i \}_1^n), j=1, 2 \}.$$
Note that $A$ is the $n-1$ dimensional subspace parallel to the affine space $\text{aff}(\{\mathbf{x}_i \}_1^n)$ generated by the data  i.e. $A$ contains the origin. 

We first show that without loss of generality $d = n -1$. Note that both $A$ and $P$ are invariant to a fixed translation of the data. Therefore, we may translate the data so that $0 \in \text{aff}(\{\mathbf{x}_i\}_{1}^n)$ (e.g. translate by the mean of the data). The data now span an $n-1$ dimensional subspace since the affine hull of the data now contains the origin. Furthermore, $\text{span}(\{\mathbf{x}\}_{1}^n)= \text{aff}(\{\mathbf{x}\}_{1}^n) = A$. Thus without loss of generality we may consider the data to in fact be $n-1$ dimensional (i.e. $d = n -1$). 

We are now looking for a vector $\mathbf{v} \in A$ that gives complete data piling. Note by the above discussion and assumption we have $A = \mathbb{R}^d$. This means we are looking for $\mathbf{v} \in  \mathbb{R}^d$ and $a, b \in \mathbb{R}$ with $a \neq 0$ satisfying the following $n$ linear equations
$$\mathbf{x}_i^T \mathbf{v} = a y_i + b \text{ for } i = 1, \dots, n.$$ 
Since the magnitude of $\mathbf{v}$ is arbitrary we fix $a = 1$ without loss of generality. We now have
$$\mathbf{x}_i^T \mathbf{v} = y_i + b \text{ for } i = 1, \dots, n$$ 
which can be written in matrix form as
\begin{equation}
\label{eq:dp_sum}
X\mathbf{v} + b\mathbf{1}_n = \mathbf{y}
\end{equation}
where $X \in \mathbb{R}^{n \times d}$ is the data matrix whose rows are the data vectors $\mathbf{x}_i$ and $\mathbf{y} \in \mathbb{R}^n$ is the vector of class labels. This is a system of $n$ equations in $\mathbf{R}^{d + 1}$ which can be seen by appending 1 onto the end of each $\mathbf{x}_i$ i.e. $\tilde{\mathbf{x}}_i = (\mathbf{x}_i, 1) \in \mathbb{R}^{d +1}$ and letting $\mathbf{w} = (\mathbf{v}, b)$. Then Equation \ref{eq:dp_sum} becomes
\begin{equation}
\label{eq:dp_matrix}
\widetilde{X}\mathbf{w} = \mathbf{y}
\end{equation}
where $\widetilde{X} \in \mathbb{R}^{n \times d+1}$ is the appended data matrix. 

Recall that we assumed $d = n-1$ so Equation \ref{eq:dp_matrix} is a system of $n$ equations in $\mathbb{R}^n$. Further recall that the data are in general position meaning that the $n$ data points are affine independent in the $n-1$ dimensional subspace of the data. Affine independence is equivalent to linear independence of $\{(\mathbf{x}_i, 1)\}_1^n$. Therefore the matrix $\widetilde{X} \in \mathbb{R}^{n \times n}$ has full rank and Equation \ref{eq:dp_matrix} always has a solution, $\mathbf{v}^*$, and this solution is unique. 

Existence of a solution to Equation \ref{eq:dp_matrix} shows that $P \cap A \neq \emptyset$. Uniqueness of the solution to Equation \ref{eq:dp_matrix} shows that this intersection $P \cap A$ can have only one direction of which $\mathbf{v}^*$ is a representative element. 

We now show that $\mathbf{v}^*$ is in fact the maximal data piling direction. We no longer assume that $d = n -1$. 

We first construct an orthonormal basis $\{\mathbf{t}_i\}_1^d$ of $\mathbb{R}^d$as follows. Let the first $n -1$ basis vectors $\mathbf{t}_1, \dots, \mathbf{t}_{n-1}$ span $A$. Let $\mathbf{t}_{n}$ be orthogonal to $A$ but in the span of the data $\{\mathbf{x}_i\}_1^n$ (recall the data span an $n$ dimensional space while the affine hull of the data is $n-1$ dimensional). Let the remaining $d - n +1$ basis vectors be orthogonal to $A$ and the span of the data. 

We show that the vector $\mathbf{t}_{n}$ projects every data point onto a single point i.e. $\mathbf{x}_{i}^T\mathbf{t}_{n} = c$ for each $i=1, \dots n$ and some $c \in \mathbb{R}$. Suppose we translate $\text{aff}(\{\mathbf{x}_i\}_{1}^n)$ along $\mathbf{t}_{n}$ until the origin lies in the affine hull of the translated data. In particular, the data now span an $n-1$ dimensional subspace that is orthogonal to $\mathbf{t}_{n}$ (where as before they spanned an $n$ dimensional subspace). We now have that for some $c \in \mathbb{R}$
$$\mathbf{t}_{n}^T (\mathbf{x}_i +c \mathbf{t}_{n})= 0 \text{ for each } i= 1, \dots, n$$
$$\mathbf{t}_{n}^T \mathbf{x}_i = c \text{ for each } i= 1, \dots, n$$
since $\mathbf{t}_{n}$ is unit norm.

Let $\mathbf{v} \in \mathbb{R}^d$ be a representative vector of the direction in the affine hull of the data that gives complete data piling (given above). Suppose $\mathbf{v}$ has unit norm and is oriented such that 
$$\mathbf{v}^T \mathbf{x}_i = a y_i + b$$
for some $a, b \in \mathbb{R}$ with $a >0$ (note fixing $a >0$ eliminates the antipodal symmetry of data piling vectors).

We now show that $\mathbf{v}$ is in fact the maximal data piling direction. Let $\mathbf{w} \in \mathbb{R}^d$ be another vector with unit norm that gives complete data piling (i.e. $\mathbf{w} \in P$). In particular, there exists $a_v, a_w, b_v, b_w  \in \mathbb{R}$ with $a_v, a_w > 0$ such that
$$\mathbf{v}^T \mathbf{x}_i = a_v y_i + b_v \text{ for each } i = 1, \dots, n.$$
$$\mathbf{w}^T \mathbf{x}_i = a_w y_i + b_w \text{ for each } i = 1, \dots, n.$$

Assume for the sake of contradiction that $\mathbf{w}$ projects the data possibly further apart than $\mathbf{v}$ does. In particular assume that $a_w \ge a_v$.

Since $\{\mathbf{t}_i\}_1^d$ is a basis we can write
$$\mathbf{w} = \sum_{i=1}^d \alpha_i \mathbf{t}_i.$$
Next compute the dot products with the data. For any $j = 1, \dots, n$, 
$$\mathbf{w} ^T \mathbf{x}_j =  \left(\sum_{i=1}^{n-1} \alpha_i \mathbf{t}_i \right)^T \mathbf{x}_j + \alpha_n \mathbf{t}_n^T\mathbf{x}_j +  \sum_{i=n+1}^d \alpha_i \mathbf{t}_i^T \mathbf{x}_j.$$ 
Recall the basis vectors $\mathbf{t}_{n+1}, \dots, \mathbf{t}_d$ are orthogonal to the data points so the third term in the sum is zero. Furthermore, the dot product of $\mathbf{t}_n$ with each data point is a constant. Thus we now have
$$\mathbf{w}^T \mathbf{x}_j =  \left(\sum_{i=1}^{n-1} \alpha_i \mathbf{t}_i \right)^T \mathbf{x}_j  + \alpha_n c, \text{ for all } j=1, \dots, n.$$
Thus we can see the vector
$$\mathbf{w}' = \sum_{i=1}^{n-1} \alpha_i \mathbf{t}_i $$
also gives complete data piling. However this vector lies in $A$ since it is a linear combination of the first $n-1$ basis vectors. We have shown that there is only one direction in $A$ with complete data piling thus $\sum_{i=1}^{n-1} \alpha_i \mathbf{t}_i \propto \mathbf{v}$. In particular, for some $\alpha > 0$
$$\sum_{i=1}^{n-1} \alpha_i \mathbf{t}_i = \alpha \mathbf{v}.$$
So we now have
$$\mathbf{w}'  = \alpha \mathbf{v} + \alpha_n \mathbf{t}_n.$$

Recall $||\mathbf{v}|| = || \mathbf{w}|| = 1$ and $\mathbf{t}_n$ is orthogonal to $\mathbf{v}$ by construction. Therefore $\alpha^2 + \alpha_n^2 = 1$. In particular if $\alpha_n >0$ then $\alpha < 1$. 

Let $\mathbf{x}_+$ and $\mathbf{x}_-$ be any point from the positive and negative class respectively. By construction we have
$$\mathbf{v}^T (\mathbf{x}_+ - \mathbf{x}_-) = a_v.$$
$$\mathbf{w}^T (\mathbf{x}_+ - \mathbf{x}_-) = a_w.$$
However expanding this last line we get
$$\mathbf{w}^T (\mathbf{x}_+ - \mathbf{x}_-) = ( \alpha \mathbf{v} + \alpha_n \mathbf{t}_n)^T (\mathbf{x}_+ - \mathbf{x}_-) $$
$$\mathbf{w}^T (\mathbf{x}_+ - \mathbf{x}_-) = \alpha \mathbf{v}^T (\mathbf{x}_+ - \mathbf{x}_-) + \alpha_n \mathbf{t}_n^T (\mathbf{x}_+ - \mathbf{x}_-).$$
But $\mathbf{t}_n^T \mathbf{x}_+  = \mathbf{t}_n^T \mathbf{x}_- = c$ so the last term is zero.  Thus we now have
$$\mathbf{w}^T (\mathbf{x}_+ - \mathbf{x}_-) = \alpha a_v.$$
Thus
$$ \alpha a_v = a_w.$$
However unless $\mathbf{w} = \mathbf{v}$ (so $\alpha_n = 0$) we have $0< \alpha <1$. Therefore $a_w < a_v$ contradicting the assumption that $a_w \ge a_v$. Therefore $ \mathbf{v}$ is the maximal data piling direction.
\end{proof}


\section*{Appendix B.} \label{app:hm_mdp_linprog}
Theorem \ref{thm:hm_mdp} gives a geometric characterization when the set of convex directions intersects the set of complete data piling directions. We can also characterize this event through a linear program. 

An alternative way of deciding if $C \cap P = \emptyset$ and computing the intersection if it exists is through the following linear program (proof of Theorem \ref{bp_lp_equiv} is a straightforward exercise in linear programming).
\begin{theorem}
\label{bp_lp_equiv}
$C \cap P \neq \emptyset$ if and only if there is a solution to the following linear program

\begin{equation}
\label{eq:bp_lp}
\begin{aligned}
& \underset{\alpha \in \mathbb{R}^{n_+},\beta \in \mathbb{R}^{n_-}, \mathbf{v} \in \mathbb{R}^d, b \in \mathbb{R}}{\text{minimize}} & & 1 \\
& \text{subject to} & &   X\mathbf{v} + \mathbf{1}_n b = \mathbf{y} \\
& & & \sum_{i \in I_+} \alpha_i \mathbf{x}_i  - \sum_{i \in I_-} \beta_i \mathbf{x}_i  = \mathbf{v} \\
& & &  \sum_{i \in I_+} \alpha_i  = 1 \\
& & & \sum_{i \in I_-} \beta_i  = 1\\ 
& & &\alpha_i, \beta_i \ge 0 \text{ for } i = 1, \dots, n.
\end{aligned}
\end{equation}
In the case a solution $\mathbf{v}$ exists then $\mathbf{v} \in C \cap P$. 
\end{theorem}
The vector $\mathbf{1}_n \in \mathbb{R}^n$ is the vector of ones, $X$ is the $\mathbb{R}^{n \times d}$ data matrix and $\mathbf{y} \in \mathbb{R}^n$ is the vector of class labels. The first constraint says $\mathbf{v}$ must be a complete data piling direction, $\mathbf{v} \in P$. The remaining constraints say $\mathbf{v}$ must be a convex direction, $\mathbf{v} \in C$.

Note that solving this linear program is at least as hard as solving the original SVM quadratic program therefore Theorem \ref{bp_lp_equiv} is not of immediate computational interest. This theorem, however, does give an alternate mathematical description $C \cap P \neq \emptyset$ which may be of theoretical interest.

\vskip 0.2in
\bibliography{svm_refs}

\end{document}